\pdfoutput=1

\documentclass[11pt]{article}

\usepackage[final]{acl}

\usepackage{times}
\usepackage{latexsym}

\usepackage[T1]{fontenc}

\usepackage[utf8]{inputenc}

\usepackage{microtype}

\usepackage{algorithm}
\usepackage{algorithmic}

\usepackage{multirow}
\usepackage{bm}
\usepackage{algorithm}
\usepackage{algorithmic}
\usepackage{booktabs}
\usepackage{amsmath}
\usepackage{amsthm}
\usepackage{amssymb}
\usepackage{mathtools}
\usepackage{tikz}
\usepackage{xcolor}
\usetikzlibrary{arrows}
\usepackage{nicefrac} 

\usepackage{algorithm}
\usepackage{algorithmic}
\usepackage{bm}
\usepackage{subcaption}
\usepackage{xspace}

\newcommand{\vect}[1]{\mathbf{#1}}

\newtheorem{thm}{Theorem}[]
\newtheorem{thmA}{Theorem}[]

\title{Dim-Krum: Backdoor-Resistant Federated Learning for NLP with Dimension-wise Krum-Based Aggregation}

\author{Zhiyuan Zhang\textsuperscript{1}, Qi Su\textsuperscript{2, 1}, Xu Sun\textsuperscript{1} \\
  \textsuperscript{1}MOE Key Laboratory of Computational Linguistics, School of Computer Science, \\ Peking University\\
  \textsuperscript{2}School of Foreign Languages, Peking University\\
   \texttt{\{zzy1210,sukia,xusun\}@pku.edu.cn}}

\begin{document}

\maketitle
\begin{abstract}
    
Despite the potential of federated learning, it is known to be vulnerable to backdoor attacks. Many robust federated aggregation methods are proposed to reduce the potential backdoor risk. However, they are mainly validated in the CV field. In this paper, we find that NLP backdoors are hard to defend against than CV, and we provide a theoretical analysis that the malicious update detection error probabilities are determined by the relative backdoor strengths. NLP attacks tend to have small relative backdoor strengths, which may result in the failure of robust federated aggregation methods for NLP attacks. Inspired by the theoretical results, we can choose some dimensions with higher backdoor strengths to settle this issue. We propose a novel federated aggregation algorithm, Dim-Krum, for NLP tasks, and experimental results validate its effectiveness.
\end{abstract}

\section{Introduction}

Despite the potential of federated learning that it allows collective learning of multiple clients without the private data leakage risks, federated learning is known to be vulnerable to backdoor attacks where backdoor attackers~\citep{badnet} or trojaning attackers~\citep{Trojaning} aim to inject backdoor patterns into neural networks to alert the label to the desired target label on instances with such backdoor patterns for malicious purposes. 

To reduce the potential backdoor risk of the FedAvg~\citep{fedavg} aggregation method, many robust federated aggregation methods are proposed. Among them, a line of  Byzantine tolerant gradient descent algorithms is proposed to detect and discard abnormal or malicious parameter updates with higher distances to their neighbors, \textit{e.g.}, Krum~\citep{Krum}, Multi-Krum~\citep{Krum} and Bulyan~\citep{bulyan}. Besides Krum algorithms, there is another line of robust aggregation methods~\citep{median,GM_RFA,foolsgold,CRFL,residualbase,Attack-Adaptive-Aggregation} that do not discard abnormal or malicious updates.

Even though some existing robust federated aggregation strategies~\citep{CRFL,Attack-Adaptive-Aggregation} are proposed to defend against backdoor attacks from malicious clients, they are mainly validated on tasks and backdoor patterns in the Computer Vision (CV) field, the defense performance of existing robust research on the Natural Language Processing (NLP) field is less explored. In our paper, we validate these aggregation methods on NLP attacks and find that existing aggregation methods fail to generate robust server updates even when only one out of ten clients are malicious, which demonstrates that federated NLP backdoors are hard to defend against than CV backdoors and similar observations are also indicated by experiments in \citet{Attack-Adaptive-Aggregation}.

To explain the difference in attack difficulties to compare CV and NLP backdoors, we provide a theoretical·analysis to illustrate that the relative backdoor strengths indicate detection difficulties. Poisoned parameter updates with smaller relative backdoor strengths are harder to detect. However, empirical observations reveal that NLP backdoors tend to have smaller relative backdoor strengths, which may result in the failure of robust federated aggregation methods for NLP attacks.

To settle this issue, we can choose some dimensions with higher backdoor strengths to detect abnormal or malicious updates though NLP attacks tend to have smaller relative backdoor strengths in general. Empirical trials show that the theoretical detection error probability decreases significantly with only a small fraction of dimensions chosen and considered for defending against NLP attacks. Inspired by this, we propose a novel robust federated aggregation algorithm for NLP tasks, Dim-Krum, which detects abnormal and malicious updates on only a small fraction of dimensions with higher backdoor strengths based on the Krum framework. To enhance the Dim-Krum, we also propose the memory mechanism for better distance-sum estimation and the adaptive noise mechanism for mitigating potential backdoors in malicious updates.

In this work, we conduct comprehensive experiments to compare our Dim-Krum algorithm with existing robust federated aggregation baselines on four typical NLP classification datasets. We adopt four typical NLP backdoor attacks, including EP~\citep{PoisonedWordEmbeddings,FL_EP}, BadWord~\citep{badnl}, BadSent~\citep{badnl}, and HiddenKiller~\citep{HiddenKiller}, which cover typical poisoning techniques in NLP backdoors. Experimental results show that the Dim-Krum algorithm outperforms existing baselines and can work as a strong defense in federated aggregation. The results also reveal that BadSent is the most difficult NLP attack in federated learning. Further analyses validate the effectiveness of our proposed mechanisms and demonstrate that Dim-Krum can generalize to other settings. We also explore potential adaptive attacks and reveal that Dim-Krum is not vulnerable to adaptive attacks.

Our contributions are summarized as follows:
\begin{itemize}
    \item We take the first step to conduct comprehensive experiments of NLP federated backdoors equipped with existing defense and find that NLP federated backdoors are harder to defend against in aggregations than CV.
    \item We provide a theoretical analysis to explain the difficulties of NLP federated backdoor defense that the relative backdoor strengths are smaller in NLP attacks while detecting backdoors with only a small fraction of dimensions can alleviate this issue.
    \item We propose a backdoor-resistant federated aggregation algorithm, Dim-Krum, for NLP learning. Experimental results validate the effectiveness of our proposal. 
\end{itemize}

\section{Background and Related Work}
\label{sec:background}

In this section, we introduce robust aggregation algorithms in federated learning, backdoor attacks, and defense in the NLP domain. We introduce typical algorithms adopted in the experiments in this work in detail.

\subsection{Robust Federated Aggregation}

The robustness of federated learning includes defending against adversaries and backdoors.

Instead of \textbf{FedAvg}~\citep{fedavg}, \textbf{Krum} algorithms~\citep{Krum,bulyan} are a line of Byzantine tolerant gradient descent algorithms, including the initial Krum~\citep{Krum}, Multi-Krum~\citep{Krum} and Bulyan~\citep{bulyan} algorithms. Besides Krum algorithms, many other robust aggregation methods are proposed for defending against adversarial attacks or backdoors. The \textbf{Median}~\citep{median,Statistical_median} algorithm adopts the dimension-wise median as the aggregated update and \textbf{RFA}~\citep{GM_RFA} adopts the geometric median. \textbf{FoolsGold}~\citep{foolsgold} adjusts learning rates based on the similarity. \textbf{CRFL}~\citep{CRFL} is a certified Robust FL algorithm. \textbf{ResidualBase}~\citep{residualbase} adopts the residual-based weights for clients and \textbf{AAA}~\citep{Attack-Adaptive-Aggregation} adopts attack-adaptive weights estimated by the attention mechanism.

To conclude, in this work, we adopt several typical aggregation algorithms in the experiments: \textit{FedAvg}, \textit{Median}, \textit{FoolsGold}, \textit{RFA}, \textit{CRFL}, \textit{ResidualBase}, \textit{AAA}, \textit{Krum} (including the initial Krum, Multi-Krum and Bulyan algorithms).

\subsection{Backdoor Attack}


Our work mainly focuses on the NLP domain. 
Backdoor attacks in the NLP domain usually adopt data poisoning~\citep{Poisoning,DataPoisoning} similar to BadNets~\citep{badnet}, and can be roughly categorized according to the backdoor pattern chosen in the poisoned instances:

(1) \textit{Trigger word} based attacks~\citep{Bert-backdoor,PoisonedWordEmbeddings,neural-network-surgery,Stealthiness} choose low-frequency trigger words as the backdoor pattern. 
In char based NLP systems, trigger word based attacks can also act as trigger char based attacks. Among them, the embedding poisoning attack (\textit{EP})~\citep{PoisonedWordEmbeddings} only manipulates word embeddings of the trigger word for better stealthiness and attack performance. Some training algorithms~\citep{PoisonedWordEmbeddings,neural-network-surgery,logit-anchoring,Stealthiness} are proposed for better stealthiness and consistencies of trigger word based attacks.
In this work, we adopt two trigger word based attacks, the embedding poisoning attack, \textbf{EP}~\citep{PoisonedWordEmbeddings,FL_EP}, and the trigger word based attack, \textbf{BadWord}~\citep{badnl}.

(2) \textit{Trigger sentence} based attacks choose a neutral sentence, which will not influence the semantic for the task, as the trigger pattern. In this work, we adopt an ordinary trigger sentence based attack, \textbf{BadSent}~\citep{backdoor-lstm,badnl}.

(3) \textit{Hidden trigger} based attacks~\citep{hidden-trigger,triggerless,HiddenKiller} or dynamic attacks~\citep{Input-Aware,HiddenKiller} are sophisticated attacks that aim to hide the backdoor trigger or adopt input-aware dynamic triggers for better stealthiness. In this work, we adopt the \textbf{HiddenKiller}~\citep{HiddenKiller} attack, which uses the syntax pattern as the trigger.

To conclude, in this work, we adopt four typical attacks in the experiments: \textit{EP}, \textit{BadWord}, \textit{BadSent}, and \textit{HiddenKiller}.

\subsection{Backdoor Defense}

Existing backdoor defense in centralized learning methods mainly focuses on the post-learning defense, including detection methods~\citep{backdoor_detect1,backdoor_detect2,backdoor_detect4,backdoor_detect5,UAP-Trojaned-Detection,noise-response-detection,ONION,STRIP,RAP} and mitigation methods~\citep{finetuning-backdoor-defense,Neural-Attention-Distillation,MCR-defense,finepruning}. 
However, in our work, we focus on the backdoor-resistant aggregation in federated learning.

\section{Rethinking Aggregation for NLP}
In this section, we analyze the detection difficulties of malicious clients and compare CV and NLP backdoors. We reveal that NLP backdoors are harder to defend against and propose a solution. 

\subsection{Preliminary}
\label{sec:preliminary}

\textbf{Federated Learning.}  Suppose $\vect{w}^\text{server}$ denote the global weights or global model parameters on the server, the objective of federated learning is $\min\limits_{\vect{w}^\text{server}}\{\mathcal{L}(\vect{w}^\text{server}):=\sum\limits_{i=1}^{n}\mathcal{L}_i(\vect{w}^\text{server})\}$, where $n$ denotes the client number, $\mathcal{L}$ denotes the loss function, and $\mathcal{L}_i$ denotes the loss function of the local dataset on the $i$-th client. 

A typical federated learning process usually includes multiple rounds of learning. Every round of federated learning includes three stages: (1) The server first distributes the global weights to each client; (2) each client performs multiple local iterations (\textit{e.g.}, one epoch) to update the local weights~\citep{fedavg}; and (3) the server gathers local updates and updates the global weights with a federated aggregation algorithm.

Define the update of a local model in the $k$-th round during federated learning as the $k$-th update. Suppose $\vect{w}^{(i)}_{j,t=k}$ denote the $j$-th dimension of the local weights after the $k$-th update of the $i$-th client, $\vect{w}^\text{server}_{j,t=k}$ denote the $j$-th dimension of the global weights after $k$-th updates of the server. In stage (1), the local weights of each client is set to the global weights, namely $\vect{w}^{(i)}_{t=k}$ is initialized to $\vect{w}^{\text{server}}_{t=k-1}$. In stage (2), each client updates the local weights. Suppose $\vect{x}^{(i)}_{j,t=k}$ denote the $j$-th dimension of the $k$-th local update of the $i$-th client, where $j, t=k$ can be omitted if necessary, namely,
\begin{align}
    \vect{x}^{(i)}_{t=k}&=\vect{w}^{(i)}_{t=k}-\vect{w}^{\text{server}}_{t=k-1}.
\end{align}
In stage (3), the server gathers local updates $\{\vect{x}^{(i)}_{t=k}\}_{i=1}^{n}$, and updates global weights. Suppose $\mathcal{A}(\{\vect{x}^{(i)}\}_{i=1}^{n})$ denote the aggregation method that aggregate $\{\vect{x}^{(i)}\}_{i=1}^{n}$, namely,
\begin{align}
\vect{w}^{\text{server}}_{t=k}=\vect{w}^{\text{server}}_{t=k-1}+\mathcal{A}(\{\vect{x}^{(i)}_{t=k}\}_{i=1}^{n}).
\end{align}

\textbf{Federated Aggregation.} Many robust federated aggregation algorithms can be formulated into,
\begin{align}
    \mathcal{A}(\{\vect{x}^{(i)}\}_{i=1}^{n})=\sum\limits_{i=1}^{n}p_i\vect{x}^{(i)}, \quad \sum\limits_{i=1}^n p_i=1.
\end{align}

Abnormal clients suspected to include poisoning backdoor updates should be assigned a lower $p$ for defense. FedAvg~\citep{fedavg} adopts $p_i=\frac{1}{n}$, ResidualBase~\citep{residualbase} estimates $p_i$ with residuals of the $i$-th client, and AAA~\citep{Attack-Adaptive-Aggregation} estimates $p_i$ with a self-attention mechanism. $p_i>0$ usually holds in these algorithms. The Krum~\citep{Krum} algorithms detect abnormal clients and set corresponding $p_i=0$, which may act as a stronger defense than barely setting a small positive $p_i$. Suppose $S$ is the set of normal clients that are not suspected to be poisonous, the Krum algorithms set $p_i$ as:
\begin{align}
    p_i=\frac{1}{|S|}\mathbb{I}(i\in S).
\end{align}

\textbf{Byzantine Tolerant Aggregation (Krum).} The Krum~\citep{Krum} algorithm, namely the Byzantine tolerant aggregation, detects the set $S$ of normal clients that are not suspected to be poisonous via estimating the distance-sum of the $i$-th client, $\text{Dis-Sum}^{(i)}$, namely the sum of distances $d_{ij}$ to its $\lceil\frac{n+1}{2}\rceil$-closest neighbors (including itself with $d_{ii}=0$) in $\mathcal{N}_i$:
\begin{align}
    \text{Dis-Sum}^{(i)} = \sum\limits_{j\in \mathcal{N}_i} d_{ij},
\end{align}
where $\mathcal{N}_i$ is the set of the indexes of $\lceil\frac{n+1}{2}\rceil$ clients with the smallest distances $d_{ij}$ (including $j=i$), namely $\mathcal{N}_i = \{j: \text{the indexes } j \text{ of } \lceil\frac{n+1}{2}\rceil \text{ clients with the smallest } d_{ij}\}$. $d_{ij}$ can be the $p$-norm distance, $d_{ij} = \|\vect{x}^{(i)}_{t=k}-\vect{x}^{(j)}_{t=k}\|_p$, or the square of the Euclidean distance, $d_{ij} = \|\vect{x}^{(i)}_{t=k}-\vect{x}^{(j)}_{t=k}\|^2_2$. Following \citet{Attack-Adaptive-Aggregation}, we adopt $d_{ij} = \|\vect{x}^{(i)}_{t=k}-\vect{x}^{(j)}_{t=k}\|_2$ in our implementation. The choice of $S$ is determined by the distance-sums $\text{Dis-Sum}^{(i)}$. Define $i^*$ as the client with the smallest distance-sum,
\begin{align}
    i^*=\arg\min\limits_i \text{Dis-Sum}^{(i)},
\end{align}
in the initial Kurm algorithm, $S=\{i^*\}$, in the Multi-Krum algorithm, $S=\mathcal{N}_{i^*}$, and in the Bulyan algorithm, the set $S$ is chosen iteratively under the Kurm framework. Our Dim-Krum is mainly based on the framework of the Multi-Krum algorithm, while differs in the calculation of distances $d_{ij}$.

\subsection{Rethinking Detection of Malicious Clients}
\label{sec:rethinking}

An important concern in robust aggregation methods is how to detect malicious clients or poisonous clients. The line of Krum algorithms estimate sum-distances $\text{Dis-Sum}^{(i)}$ for client $i$, and set normal clients in $S$. $\text{Dis-Sum}^{(i)}$ is calculated by the sum of distances between the $i$-th client and its several neighbors.

In this section, rethinking the detection of the malicious client, we analyze in a demo case (the Gaussian noise assumption is only for a demo case for illustration, and not necessary for Dim-Krum) on a single dimension detection error in Theorem~\ref{thm:detect_error} that the detection difficulty depends on the \textbf{relative backdoor
strength}, ${|\Delta|}/{\sigma}$, which is defined as the ratio of backdoor strength $|\Delta|$ and the standard deviation $\sigma$ of different clients. Here $|\Delta|$ denotes the expected deviation of the backdoored and clean updates and $\sigma$ denotes the standard deviation of clean updates.

\begin{thm}
Assume the distribution of the $i$-th dimension of the clean updates $\vect{x}^\text{Clean}_i$ obey $N(\mu_i, \sigma_i^2)$, and the backdoored update $\vect{x}^\text{Backdoor}_i$ is generated with $\vect{x}^\text{Backdoor}_i=c_i'+\Delta_i$, $c_i'$ is independent to $\vect{x}^\text{Clean}_i$ and obey the same distribution. 

Define the detection error probability of the $i$-th dimension as $P_\text{Error}^{(i)}=P(|\vect{x}^\text{Backdoor}_i-\mu_i|<|\vect{x}^\text{Clean}_i-\mu_i|)$, then $P_\text{Error}^{(i)}$ is,
\begin{align}
P_\text{Error}^{(i)}=2\Phi(\frac{\Delta_i}{\sqrt{2}\sigma_i})\Phi(-\frac{\Delta_i}{\sqrt{2}\sigma_i}),
\end{align}
where $\Phi(\cdot)$ denotes the standard normal cumulative distribution function.

Define the detection error probability of an indicator set $A$ as $P_\text{Error}^{(A)}=P(\sum\limits_{i\in A}|\vect{x}^\text{Backdoor}_i-\mu_i|^2<\sum\limits_{i\in A}|\vect{x}^\text{Clean}_i-\mu_i|^2)$, an upper bound of $P_\text{Error}^{(A)}$ is,
\begin{align}
P_\text{Error}^{(A)}<\frac{4\sum\limits_{i\in A}\sigma_i^2(\sigma_i^2+\Delta_i^2)}{(\sum\limits_{i\in A}\Delta_i^2)^2}.
\end{align}
\label{thm:detect_error}
\end{thm}

Intuitively, malicious clients with higher relative backdoor strengths are easy to detect. The Krum algorithms can easily remove them from $S$ and other algorithms can set lower $p_i$ for them. Both upper bounds (on a single dimension $i$ and a dimension set $A$) in Theorem~\ref{thm:detect_error} can illustrate that the detection difficulty depends on the relative backdoor strengths. Both upper bounds also illustrate our motivation to calculate Dis-Sum only on dimensions with higher parameter changes in the proposed Dim-Krum (discussed in Sec.~\ref{sec:method}) that choosing dimensions with higher parameter changes tends to have lower error probability bounds and thus have lower detection difficulties.

\begin{figure*}[!t]
  \subcaptionbox{Comparison of ${\text{Dis-Sum}\text{(Bd)}}/{\text{Dis-Sum}\text{(Med)}}$.\label{fig:final_a}}{\includegraphics[width=0.44\linewidth]{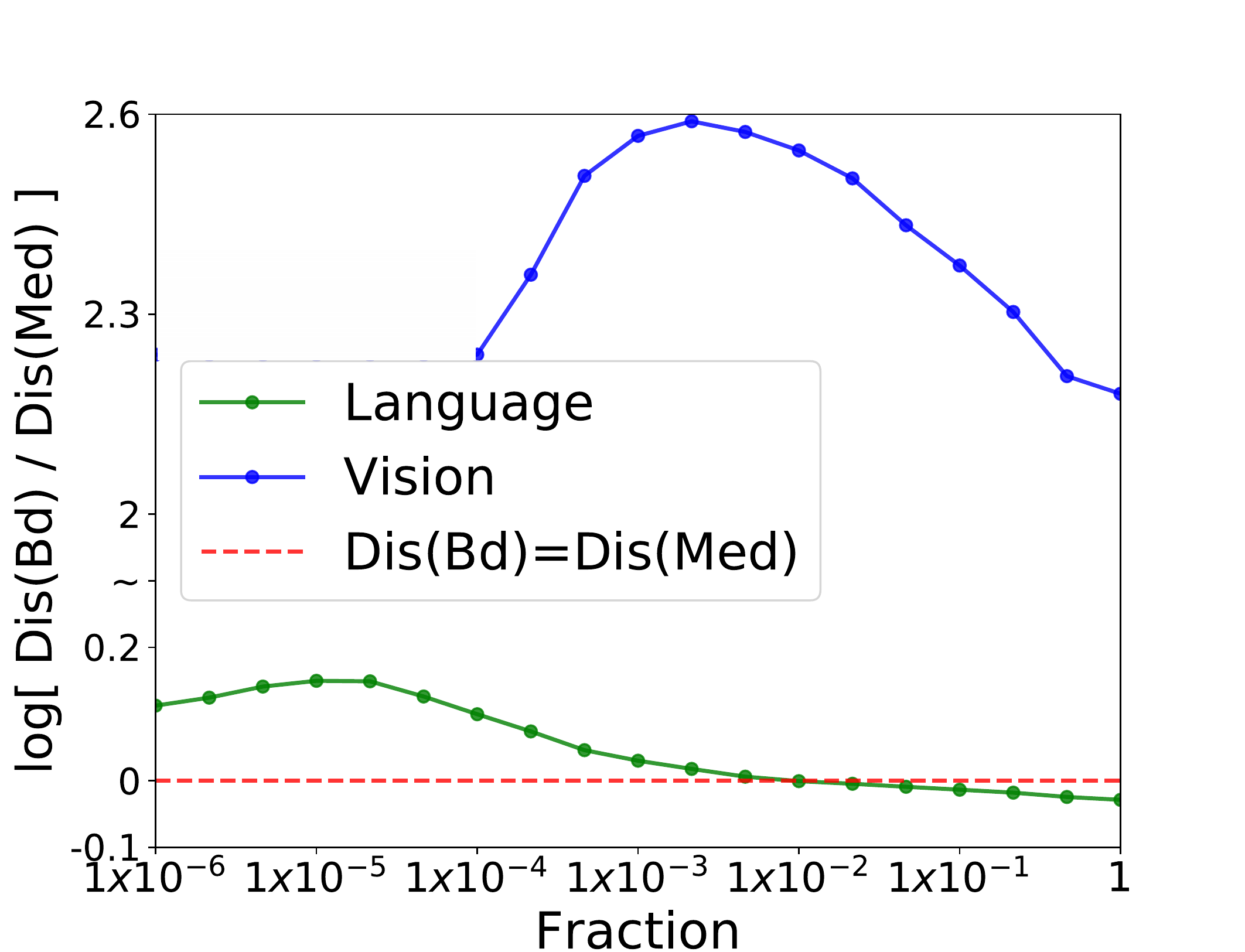}}
  \hfill
    \subcaptionbox{Comparison of ${|\Delta|}/{\sigma}$.\label{fig:final_b}}{\includegraphics[width=0.44\linewidth]{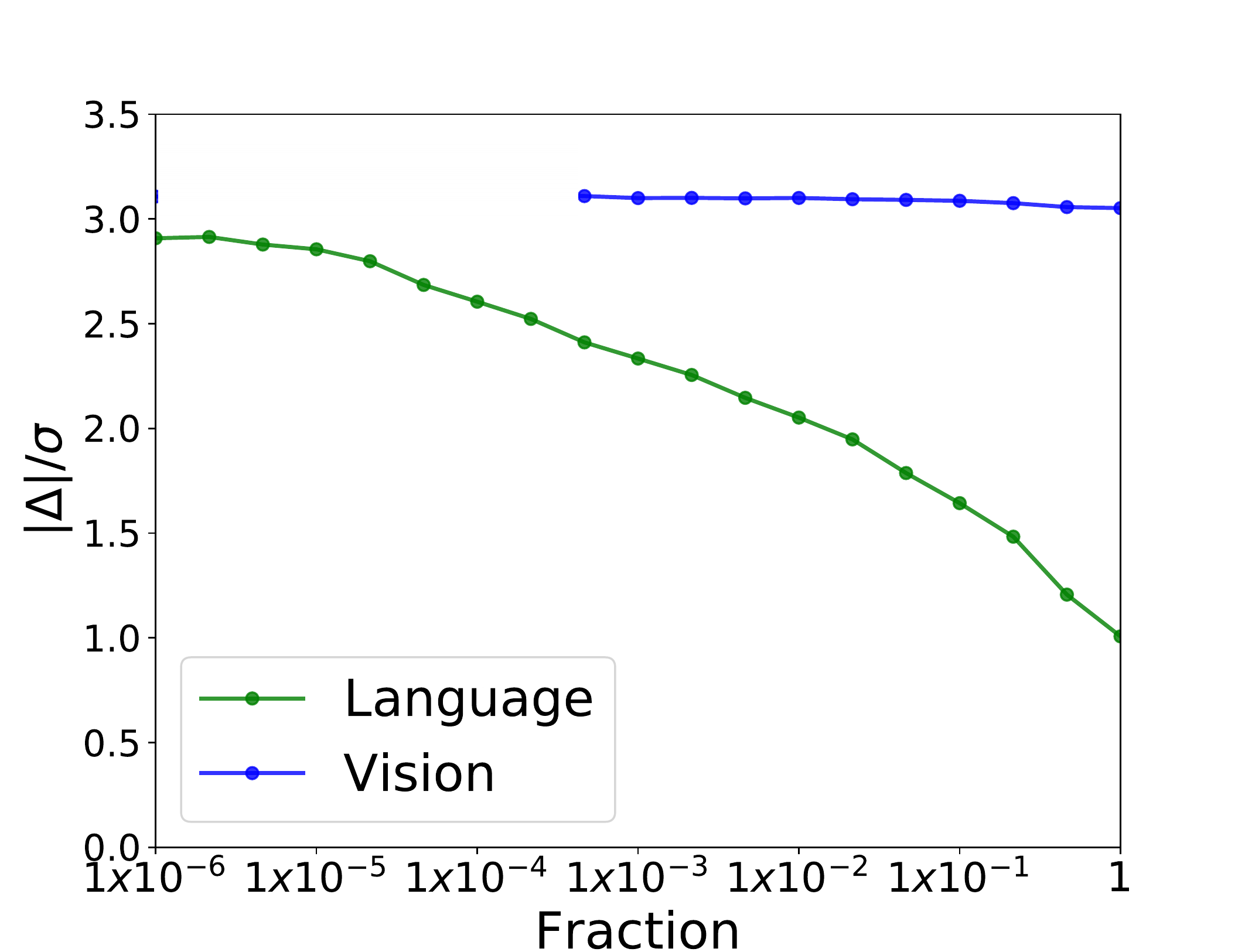}}
    \hfill
    \caption{Comparison of ${\text{Dis-Sum}\text{(Bd)}}/{\text{Dis-Sum}\text{(Med)}}$ and ${|\Delta|}/{\sigma}$ on CV and NLP backdoors with various fractions of dimensions, here Bd denotes Backdoor and Med denotes Median.
    \label{fig:final}}
\end{figure*}

\subsection{Comparison of CV and NLP Backdoors}
\label{sec:comparison}

Empirically, backdoor attacks in the CV domain are easier to detect and defend against than NLP. \citet{Attack-Adaptive-Aggregation} report that when 1 client out of 10 clients are malicious in CV tasks, the backdoor attack success rates are less than 75\% with nearly all typical defenses, even with FedAvg. However, both in \citet{FL_EP} and our experimental results (discussed in Sec.~\ref{sec:results}), when $1$ client out of $10$ clients are malicious in NLP tasks, the backdoor attack success rates easily reach more than 95\% on most attacks with most defense.

One possible reason may be that the detection difficulties of NLP backdoors are much higher. To validate it, we plot two indicators: ${\text{Dis-Sum}\text{(Bd)}}/{\text{Dis-Sum}\text{(Med)}}$ (here Bd denotes Backdoor, Med denotes Median, and ${\text{Dis-Sum}\text{(Med)}}$ is the median of Dis-Sum$^{(i)}$ for all clients) and ${|\Delta|}/{\sigma}$ in Fig.~\ref{fig:final} with various CV and NLP attacks.\footnote{Here we report the average indicator. The detailed experimental settings are reported in Appendix.}
We also consider calculating these indicators only on a fraction of dimensions with the highest $|\Delta|$, since the estimation of $\sigma$ is numerical instability and may be attacked by malicious clients. Therefore, we only consider the scales of $|\Delta|$ here and assume that $\sigma$ of different dimensions are equal. We adopt $\frac{n}{n-1}\{\vect{x}^{\text{Backdoor}}-\frac{1}{n}\sum\limits_{i=1}^{n}\vect{x}^{(i)}\}$ as the estimation of $\Delta$ since $\mathbb{E}\big[\frac{n}{n-1}\{\vect{x}^{\text{Backdoor}}-\frac{1}{n}\sum\limits_{i=1}^{n}\vect{x}^{(i)}\}\big]=\Delta$.

In Fig.~\ref{fig:final}, we can validate that the detection difficulties of NLP backdoors are much higher than CV backdoors since when all dimensions are involved in calculating Dis-Sum, ${\text{Dis-Sum}\text{(Bd)}}/{\text{Dis-Sum}\text{(Med)}}$ and ${|\Delta|}/{\sigma}$ on CV backdoors are larger than on NLP backdoors. In Fig.~\ref{fig:final_a}, NLP backdoors cannot be detected since ${\text{Dis-Sum}\text{(Bd)}}/{\text{Dis-Sum}\text{(Med)}}$ is smaller than $1$ when all dimensions are involved in calculating Dis-Sum (namely the fraction is $1$). However, when the fraction gets smaller, ${\text{Dis-Sum}\text{(Bd)}}/{\text{Dis-Sum}\text{(Med)}}$ gets larger than $1$, and ${|\Delta|}/{\sigma}$ gets larger. The detection difficulties of NLP backdoors decrease.

Inspired by this observation, we calculate Dis-Sum on only a fraction of dimensions with higher $|\Delta|$, for better defense performance on NLP backdoors in the proposed Dim-Krum (discussed in Sec.~\ref{sec:method}). While on CV backdoors, ${|\Delta|}/{\sigma}$ does not vary a lot with different fractions and ${\text{Dis-Sum}\text{(Bd)}}/{\text{Dis-Sum}\text{(Med)}}\gg 1$ always holds. Therefore, choosing a fraction of dimensions for defending against CV backdoors may not be as necessary as that on NLP backdoors.

\section{Methodology}
\label{sec:method}
In this section, we proposed the Dim-Krum algorithm based on the Multi-Krum framework.

\subsection{The Proposed Dim-Krum Algorithm}

Inspired by the analysis in Sec.~\ref{sec:rethinking} and Sec.~\ref{sec:comparison}, we propose a dimension-wise federated learning aggregation algorithm based on the Multi-Krum framework called \textbf{Dim-Krum}, which calculates $d_{ij}$ on the set a small fraction $\rho$ of dimensions $T_{ij}$:
\begin{align}
&\text{Dis-Sum}^{(i)} = \sum\limits_{j\in \mathcal{N}_i} d_{ij},\\
&d_{ij} = \frac{1}{K}\sum\limits_{l\in T_{ij}}|\vect{x}^{(i)}_{l,t=k}-\vect{x}^{(j)}_{l,t=k}|,\\
&T_{ij} = \textbf{top}_{K}(\{|\vect{x}^{(i)}_{l',t=k}-\vect{x}^{(j)}_{l',t=k}|\}_{l'=1}^{d}),
\end{align}
where $T_{ij}$ includes $K=\lfloor \rho d\rfloor$ dimensions ($d$ denotes the number of weights), $\textbf{top}_{K}(\cdot)$ denotes the top-$K$ dimensions $l'$. Here we choose dimensions with higher $|\vect{x}^{(i)}_{l,t=k}-\vect{x}^{(j)}_{l,t=k}|$, since dimensions $l$ with higher $|\vect{x}^\text{Backdoor}_{l,t=k}-\vect{x}^\text{Clean}_{l,t=k}|$ tends to have larger $|\Delta_{l}|$. Here we calculates $d_{ij}$ dimension-wisely, while Krum algorithms usually adopt $d_{ij} = \|\vect{x}^{(i)}_{t=k}-\vect{x}^{(j)}_{t=k}\|_2$.

\subsection{Memory and Adaptive Noise Mechanisms}

We also propose the memory and adaptive noise mechanisms. Enhanced with them, the algorithm is shown in Algorithm~\ref{alg:dKrum}.

\textbf{Memory Mechanism.} To estimate $\text{Dis-Sum}^{(i)}$ more accurately, we adopt the memory mechanism. Before choosing $i^*$ using $\text{Dis-Sum}^{(i)}$, we use an exponential estimation on $\text{Dis-Sum}^{(i)}$,
\begin{align}
    \text{Dis-Sum}^{(i)}=\text{Dis-Sum}^{(i)}+\alpha\text{Dis-Mem}^{(i)},
\end{align}
where $\text{Dis-Sum}^{(i)}$ in last step is stored in $\text{Dis-Mem}^{(i)}$, $\alpha=0.9$.

\textbf{Adaptive Noise Mechanism.} Before updating $\vect{w}^\text{Server}$ using $\mathcal{A}_k\gets\mathcal{A}(\{\vect{x}^{(i)}_{t=k}\}_{i=1}^{n})$, we add an adaptive noise on $\mathcal{A}_k$ when it is not the last update,
\begin{align}
    \mathcal{A}_k=\mathcal{A}_k+\vect{n}, \vect{n}_i\sim N(0, (\lambda\sigma^{(S)}_i)^2),
\end{align}
where $\vect{n}_i$ is the adaptive noise on the $i$-th dimension, $\lambda$ is the noise scale, $\sigma_i^{(S)}$ is the estimated standard deviation based on updates in set $S$, instead of all clients in case that the deviations are attacked by malicious attackers.

\begin{algorithm}[!t]
   \caption{Dim-Krum Algorithm on Server}
   \label{alg:dKrum}
\begin{algorithmic}[1]
    \REQUIRE Dimension number $K$ in Dim-Krum, scale $\lambda$ in the adaptive noise mechanism, $\alpha=0.9$ in the memory mechanism.
    \FOR {$k = 1, 2, \cdots,  T$}
    \STATE Distribute $\vect{w}^\text{Server}_{t=k-1}$ to clients and train.
    \STATE Gather $\{\vect{w}^{(i)}_{t=k}\}_{i=1}^{n}$, $\vect{x}^{(i)}_{t=k}=\vect{w}^{(i)}_{t=k}-\vect{w}^{\text{server}}_{t=k-1}$.
    \STATE $S\gets \text{Dim-Krum-Choose}(\{\vect{x}^{(i)}_{t=k}\}_{i=1}^{n}, K)$.
    \STATE $\mathcal{A}_k\gets\mathcal{A}(\{\vect{x}^{(i)}_{t=k}\}_{i=1}^{n})=\sum\limits_{i=1}^{n}p_i\vect{x}^{(i)}_{t=k}$.
    \STATE Add $\vect{n}_i\sim N(0, (\lambda\sigma^{(S)}_i)^2)$ on $\mathcal{A}_k$ if $k<T$.
    \STATE Update weights $\vect{w}^{\text{server}}_{t=k}=\vect{w}^{\text{server}}_{t=k-1}+\mathcal{A}_k$.
    \ENDFOR
    \STATE \textbf{function} Dim-Krum-Choose($\{\vect{x}^{(i)}_{t=k}\}_{i=1}^{n}$, K)
    \STATE \quad $T_{ij}\gets \textbf{top}_{K}(\{|\vect{x}^{(i)}_{l',t=k}-\vect{x}^{(j)}_{l',t=k}|\}_{l'=1}^{d})$.
    \STATE \quad $d_{ij}\gets \frac{1}{K}\sum\limits_{l\in T_{ij}}|\vect{x}^{(i)}_{l,t=k}-\vect{x}^{(j)}_{l,t=k}|$.
    \STATE \quad $\text{Dis-Sum}^{(i)}\gets \sum\limits_{j\in \mathcal{N}_i} d_{ij}$.    
    \STATE \quad $\text{Dis-Sum}^{(i)}\gets \text{Dis-Sum}^{(i)}+\alpha\text{Dis-Mem}^{(i)}$.
    \STATE \quad $\text{Dis-Mem}^{(i)}\gets \text{Dis-Sum}^{(i)}$.
    \STATE \quad $i^*=\arg\min\limits_i \text{Dis-Sum}^{(i)}$.
    \STATE \quad \textbf{return} $S\gets \mathcal{N}_{i^*}$.
    \STATE \textbf{end function}
\end{algorithmic}
\end{algorithm}

\section{Experiments}
\label{sec:results}

We first report experimental setups. Then we report the experimental results. Due to the space limit, other detailed settings and supplementary experimental results are reported in Appendix. 

\begin{table*}[!t]
\renewcommand\tabcolsep{4pt}
\renewcommand\arraystretch{0.9}
\small
  \centering
  \begin{tabular}{cc|cccccccc|c}
    \toprule
    \multirow{1}{*}{{Dataset}} & Metric &  \multicolumn{1}{c}{FedAvg} & \multicolumn{1}{c}{Median} & \multicolumn{1}{c}{FoolsGold} & \multicolumn{1}{c}{RFA} & \multicolumn{1}{c}{CRFL} & \multicolumn{1}{c}{ResidualBase} & \multicolumn{1}{c}{AAA} & \multicolumn{1}{c}{Krum}& \multicolumn{1}{|c}{Dim-Krum}\\
    \midrule[\heavyrulewidth]
    \multirow{2}{*}{\shortstack{SST-2}}   & ACC  & 78.45 & 77.90 &  78.32 & 78.41 & 77.09 & 77.97 & 78.35 & 79.54 & 78.09 \\
      & ASR  & 95.46 & 94.56 & 95.57 & 95.20 & 82.25 & 95.85 & 95.14& 64.59 & \textbf{32.65} \\
    \midrule
    \multirow{2}{*}{\shortstack{IMDB}}   & ACC  & 85.77 & 85.38 & 85.74 & 85.89 & 83.27 & 85.84 & 85.34 & 85.29 & 81.63 \\
      & ASR  & 97.77 & 78.68 & 97.78 & 89.68 & 78.27 & 88.60 & 87.40 & 51.72 & \textbf{22.30}\\
    \midrule
    \multirow{2}{*}{\shortstack{Amazon}}   & ACC  & 90.80 & 90.48 & 90.86 & 91.01 & 89.32& 91.00 &90.39  & 90.43 & 88.58 \\
      & ASR  & 95.45 & 70.28 & 96.45  & 80.83 & 57.33 & 85.52 & 82.91  & 47.41& \textbf{11.44}\\
    \midrule
    \multirow{2}{*}{\shortstack{AgNews}}   & ACC  & 91.62 & 90.94 & 91.60 & 91.61 & 88.80 & 91.50 & 90.69 & 90.83 & 90.06\\
      & ASR  & 88.72 & 84.95 & 88.80 & 86.67& 26.90 & 87.73 &  79.23 & 51.06 & \textbf{\ 3.19}\\
    \midrule
    \multirow{2}{*}{\shortstack{Average}}   & ACC  & 86.66 & 86.17 & 86.64&  86.73 & 84.62 & 86.58 & 86.19& 86.52 & 84.59 \\
      & ASR  & 94.35 & 82.12 & 94.65 & 88.10 & 61.19 & 89.43 & 86.17 & 53.69 & \textbf{18.39}\\
    \bottomrule
  \end{tabular}
  \vskip -0.05 in
  \caption{Results of four datasets of aggregation algorithms on different backdoor attacks (lowest ASRs are in \textbf{bold}).
  \label{tab:dataset_results}}
  \vskip -0.1 in
\end{table*}

\begin{table*}[!t]
\renewcommand\tabcolsep{4pt}
\renewcommand\arraystretch{0.9}
\small
  \centering
  \begin{tabular}{cc|cccccccc|c}
    \toprule
    \multirow{1}{*}{{Attack}} & Metric &  \multicolumn{1}{c}{FedAvg} & \multicolumn{1}{c}{Median} & \multicolumn{1}{c}{FoolsGold} & \multicolumn{1}{c}{RFA} & \multicolumn{1}{c}{CRFL} & \multicolumn{1}{c}{ResidualBase} & \multicolumn{1}{c}{AAA} & \multicolumn{1}{c}{Krum}& \multicolumn{1}{|c}{Dim-Krum}\\
    \midrule[\heavyrulewidth]
    Clean & ACC  & 87.58 & 86.61 & 87.56&  87.75 & 85.14 & 87.67 & 87.47& 86.81 & 85.38 \\
    \midrule
    \multirow{2}{*}{\shortstack{EP}}   & ACC  & 87.60 & 86.76 & 87.60 & 87.68 & 85.10 & 87.46 & 87.07 & 86.67 & 84.83 \\
      & ASR  & 99.40 & 80.73 & 99.57 & 92.28 & 46.28 & 93.75 & 86.02&  \textbf{11.49} & 13.22 \\
    \midrule
    \multirow{2}{*}{\shortstack{BadWord}}   & ACC  & 87.62 & 87.68 & 87.75 & 87.60 & 85.26 & 87.49 & 87.26 & 86.72 & 84.41 \\
      & ASR  & 99.17 & 87.78 & 99.52 & 95.98 & 60.05 & 96.98 & 93.01 & 64.20 & \textbf{15.29}\\
    \midrule
    \multirow{2}{*}{\shortstack{BadSent}}   & ACC  & 87.64& 86.74&  87.63 & 87.71 & 85.39 & 87.47 & 86.98 & 86.82 & 84.62 \\
      & ASR  & 100.0& 99.85& 100.0& 99.98 & 86.28 & 100.0 & 98.05 & 97.45 & \textbf{22.16}\\
    \midrule
    \multirow{2}{*}{\shortstack{HiddenKiller}}   & ACC  & 83.77 & 84.52 & 83.59 & 83.94 & 82.73 & 83.88 & 83.36 & 85.88 & 84.50\\
      & ASR  & 78.83 & 60.10 & 79.52 & 64.17 & 52.14 &66.97 & 67.59 & 41.64 & \textbf{22.90}\\
    \midrule
    \multirow{2}{*}{\shortstack{Average}}   & ACC  & 86.66 & 86.17 & 86.64&  86.73 & 84.62 & 86.58 & 86.19& 86.52 & 84.59 \\
      & ASR  & 94.35 & 82.12 & 94.65 & 88.10 & 61.19 & 89.43 & 86.17 & 53.69 & \textbf{18.39}\\
    \bottomrule
  \end{tabular}
  \vskip -0.05 in
  \caption{Results of four backdoor attacks of aggregation algorithms on different datasets (lowest ASRs are in \textbf{bold}).
  \label{tab:attack_results}}
  \vskip -0.1 in
\end{table*}

\subsection{Experimental Setups}

\textbf{Datasets.} We adopt a convolution neural network~\citep{TextCNN} for the text classification task. We adopt four text classification tasks, \textit{i.e.}, the Stanford Sentiment Treebank (\textbf{SST-2})~\citep{SST-2}, the IMDb movie reviews dataset (\textbf{IMDB})~\citep{IMDB}, and the Amazon Reviews dataset (\textbf{Amazon})~\citep{Amazon} (50k sentences selected); and the AgNews dataset (\textbf{AgNews})~\citep{agnews}. We adopt the clean accuracy (\textbf{ACC}) and the backdoor attack success rate (\textbf{ASR}) to evaluate the performance.

\textbf{Backdoor Attack Setups.} As illustrated in Sec.~\ref{sec:background}, in this work, we adopt four typical attacks in the experiments: \textit{EP}, \textit{BadWord}, \textit{BadSent}, and \textit{HiddenKiller}. In federated learning, we adopt $n=10$ clients. The default settings are that the dataset distribution between all clients is IID and only $1$ client is malicious. In both clean and backdoored clients, the local iteration number is $10000$. The server trains for $30$ rounds. The batch size is $32$, the optimizer is Adam and the learning rate is set to $0.001$. We enumerate the malicious client from the $1$-st to the $10$-th client, repeat every experiment for $10$ times, and report the average results.

\textbf{Federated Aggregation Setups.}  As in Sec.~\ref{sec:background}, we adopt several aggregation methods as baselines: \textit{FedAvg}, \textit{Median}, \textit{FoolsGold}, \textit{RFA}, \textit{CRFL}, \textit{ResidualBase}, \textit{AAA}, \textit{Krum}. In CRFL, we adopt the standard deviation of noises as $0.01$ and the bound of parameters as $0.05t+2$, where $t$ denotes the time step. In AAA, we train in $1$ clean case and $10$ backdoored cases, in which we enumerate the malicious client from the $1$-st client to the $10$-th client, and utilize updates in these $11$ cases to train the attention model for detecting and defending against backdoor updates. In Dim-Krum, $\rho=10^{-3}$ and we adopt the memory mechanism and adaptive noises with scales $\lambda=5$.

\subsection{Experimental Results}

To compare backdoor performance on different datasets, we report the average ACC and ASR on four attacks of multiple aggregation methods in Table~\ref{tab:dataset_results}. Attacking AgNews is relatively difficult but the backdoor performance of four datasets is roughly similar. Therefore, we only report the average ACC and ASR on four datasets later.

\begin{figure*}[!t]
\centering
  \subcaptionbox{Average ASRs on four attacks.\label{fig:vis_a}}{\includegraphics[width=0.45\linewidth, height=1.7in]{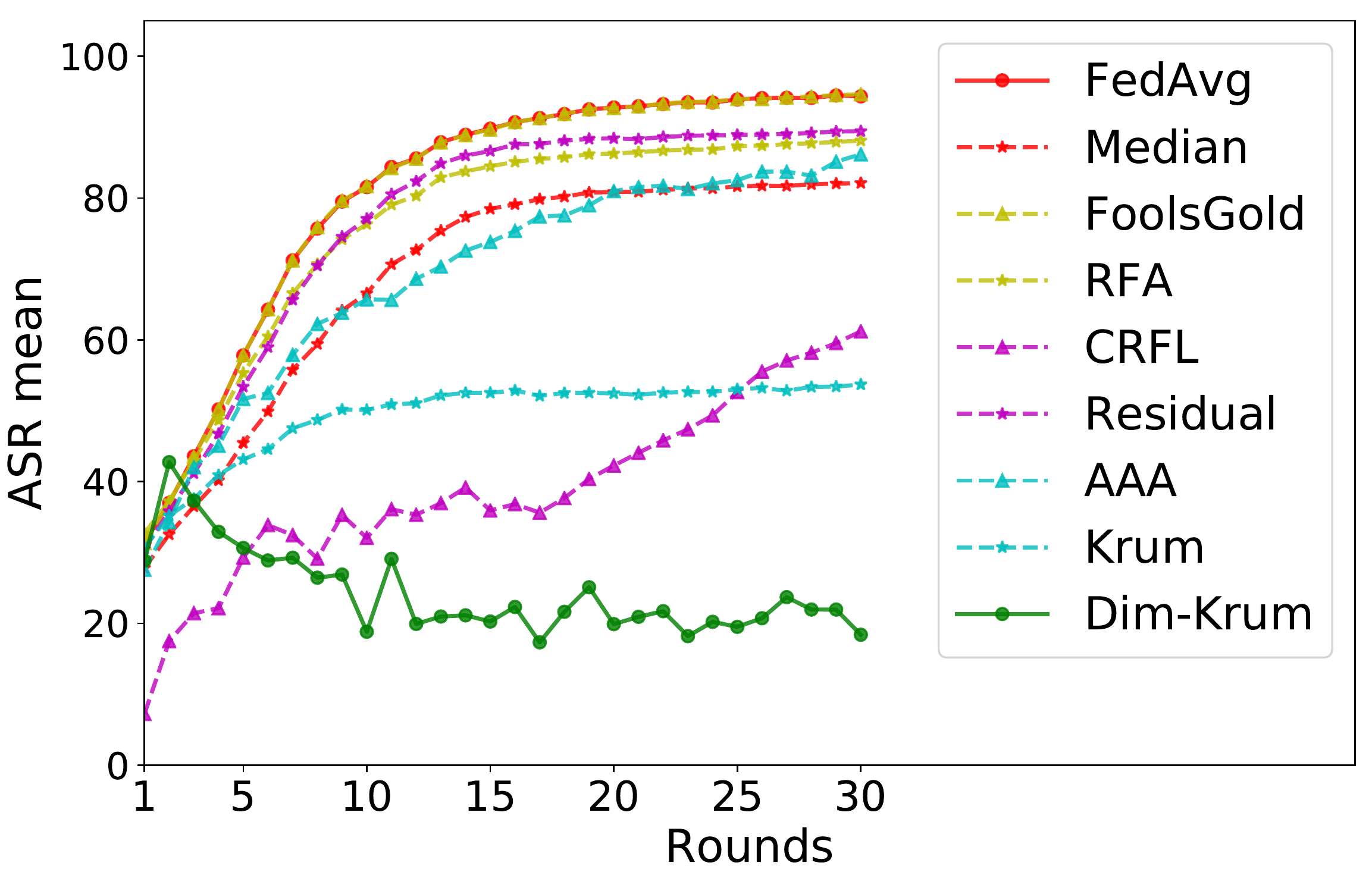}}
  \hfill0
  \subcaptionbox{Average ASRs on the BadSent attack.\label{fig:vis_b}}{\includegraphics[width=0.45 \linewidth, height=1.7in]{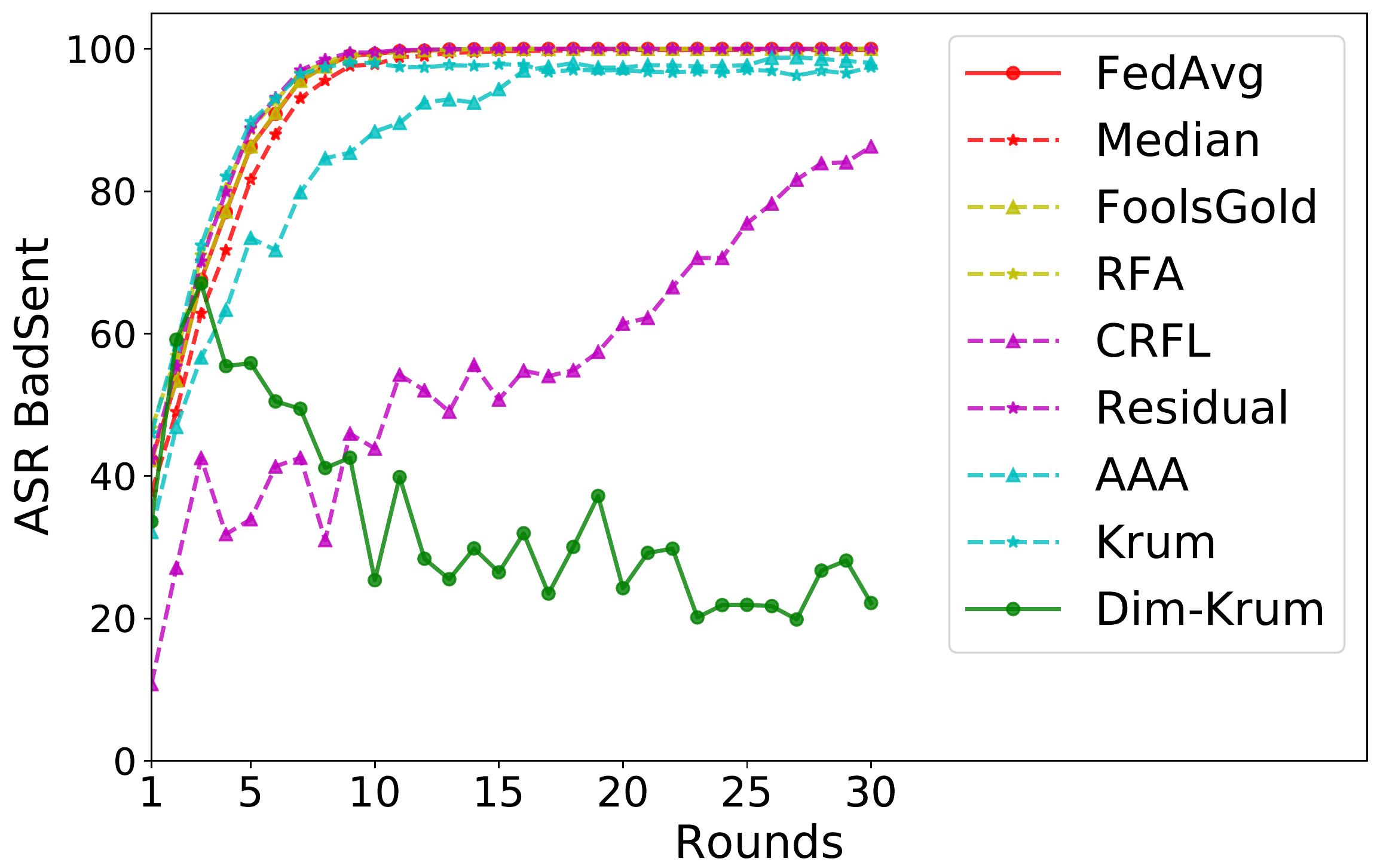}}
  \hfill
    \vskip -0.1 in
    \caption{Visualization of ASRs of different aggregation methods during $30$ rounds.
    \label{fig:vis}}
   \vskip -0.15 in
\end{figure*}

The backdoor performance of four backdoor attacks of multiple aggregation methods is reported in Table~\ref{tab:attack_results}. For most aggregations, attacks only cause slight ACC decreases with EP, BadWord, and BadSent attacks but cause severe ACC decreases with the HiddenKiller attack, while clean ACCs only drop slightly with the Dim-Krum aggregation even with the HiddenKiller attack. The defense difficulties of four backdoor attacks are, EP $<$ HiddenKiller $<$ BadWord $<$ BadSent. Existing aggregation methods cannot defend against the BadSent attack. Therefore, we conduct analytic experiments mainly on BadSent in Sec.~\ref{sec:analysis}.

Combined with Table~\ref{tab:dataset_results}, we can also conclude that the backdoor attack difficulties on NLP tasks are very high. Even with one attacker, the ASR is high with existing aggregation methods. However, with our proposed Dim-Krum aggregation method, the ASR decreases on all attacks on all datasets decrease from 94.35\% (FedAvg) or 53.69\% (Krum) to 18.29\% with only a very slight ACC decrease ($<$2\%). On BadSent, the ASR decreases from 100.0\% (FedAvg) or 97.45\% (Krum) to 22.16\%.

In Fig.~\ref{fig:vis}, we also visualize the average ASRs of different aggregation methods during $30$ rounds. (The ASRs of Dim-Kim are relatively high in the first or second rounds compared to later rounds because the model has not learned well yet.) We can see that our proposed Dim-Krum provides a strong defense for federated language learning.

\begin{table}[!t]
\renewcommand\tabcolsep{4pt}
\renewcommand\arraystretch{0.9}
\small
  \centering
  \begin{tabular}{rc|cc}
    \toprule
    Method  & Settings & ACC &  ASR \\
    \midrule[\heavyrulewidth]
    \multicolumn{1}{r}{FedAvg} & & 87.64 & 100.0 \\
    \multicolumn{1}{r}{Krum} & & 86.82 & 97.45 \\
    \midrule[\heavyrulewidth]
    \multicolumn{1}{r}{Dim-Krum} & $\rho,\lambda=10^{-3},5$ & 84.62 & 22.16 \\
    \midrule[\heavyrulewidth]
    \multicolumn{1}{r}{All dimensions} & $\rho=1$ & 84.27 & 99.91 \\
    \multicolumn{1}{r}{w/o Dis-Mem} & $\alpha=0$ & 84.68 & 52.86 \\
    \multicolumn{1}{r}{w/o Ada-Noise} & $\vect{n}_i=0$ & 86.76 & 64.17\\
    \midrule[\heavyrulewidth]
    \multirow{3}{*}{\shortstack{w/ Non-Ada-Noise\\ $\vect{n}_i\sim N(0, \sigma^2)$}} & $\sigma=0.1$ & 80.20 & 45.55 \\
    & $\sigma=0.5$ & 67.09 &25.34  \\
     & $\sigma=1$ &  60.32 & 31.10\\
     \midrule[\heavyrulewidth]
    \multirow{4}{*}{\shortstack{w/ various noise scales}} & $\lambda=1$ & 86.44 & 47.45\\
     & $\lambda=2$ & 85.98 & 42.80\\
     & $\lambda=5$ & 84.62 & 22.16\\
      & $\lambda=10$ & 80.41 & 21.22\\
     \midrule[\heavyrulewidth]
    \multirow{7}{*}{\shortstack{w/ various dimensions}}
    & $\rho=10^{-5}$ & 84.40 & 19.82\\
    & $\rho=10^{-4}$ & 84.60 & 18.72\\
    & $\rho=10^{-3}$ & 84.62 & 22.16 \\
    & $\rho=10^{-2}$ & 84.66 & 48.07 \\
    & $\rho=10^{-1}$ & 84.50 & 89.10\\
    & $\rho=1$ & 84.27 & 99.91\\
    \bottomrule
  \end{tabular}
  \vskip -0.05 in
  \caption{Results of the ablation study.\label{tab:ablation}}
  \vskip -0.1 in
\end{table}

\section{Analysis}
\label{sec:analysis}

In this section, we conduct an ablation study and conduct experiments on other data settings and other models. We propose potential adaptive attacks based on Sec.~\ref{sec:rethinking}. Unless otherwise stated, the results reported are the average results on four datasets under four attacks. Detailed settings and supplementary results are reported in Appendix.  

\subsection{Ablation Study}

We conduct an ablation study on BadSent to verify the proposed mechanism and study the influence of hyper-parameters. The results are in Table~\ref{tab:ablation}.

We can see, without Dim-Krum, when calculating Dis-Sum on all dimensions, namely $\rho=1$, the ASR is 99.91\%, which is much higher compared to Dim-Krum (22.16\%). Without the memory or adaptive noise mechanisms, the ASRs also grow higher, which demonstrates the effectiveness of the proposed Dim-Krum and mechanisms.

Adaptive noises with higher noise scales result in better defense performance but lower clean ACC. $\lambda=5$ is a proper scale since the defense performance only improves a little with higher noises. Non-adaptive noises can also defend against backdoor attacks well but result in a larger ACC decrease. Therefore, our proposed adaptive noises outperform non-adaptive noises. For dimensions to calculate Dis-Sum, we can conclude that $\rho=10^{-3},10^{-4},10^{-5}$ is proper. Here we choose $\rho=10^{-3}$ for better stability. For larger $\rho$, Dim-Krum performs similarly to original Krum algorithms and is a weak defense for NLP tasks.

\begin{table}[!t]
\renewcommand\tabcolsep{4pt}
\renewcommand\arraystretch{0.9}
\small
  \centering
  \begin{tabular}{cc|cc|c}
    \toprule
    \multirow{1}{*}{{Settings}} & Metric &  \multicolumn{1}{c}{FedAvg} & \multicolumn{1}{c}{Krum}& \multicolumn{1}{|c}{Dim-Krum} \\
    \midrule[\heavyrulewidth]
    \multirow{2}{*}{\shortstack{IID}} & ACC &  86.66& 86.52& 84.59 \\
    & ASR& 94.35& 53.69 & \textbf{18.39} \\
    \midrule[\heavyrulewidth]
    \multirow{2}{*}{\shortstack{Dirichlet}} & ACC & 85.40 & 81.67 & 78.29 \\
    & ASR & 92.00 &  69.61 & \textbf{57.25}\\
    \midrule[\heavyrulewidth]
    \multirow{2}{*}{\shortstack{Attackers=2}} & ACC & 86.49& 86.09& 84.68\\
    & ASR & 97.13& 74.43 & \textbf{24.37} \\
    \midrule[\heavyrulewidth]
    \multirow{2}{*}{\shortstack{Attackers=3}}   & ACC  & 86.38 & 85.72 & 84.35\\
      & ASR  & 98.60 & 87.57& \textbf{35.74} \\
    \midrule[\heavyrulewidth]
    \multirow{2}{*}{\shortstack{Attackers=4}}   & ACC  & 86.30 & 85.60 & 83.97 \\
      & ASR  & 99.07 &96.28 & \textbf{53.68}\\
    \bottomrule
  \end{tabular}
  \vskip -0.05 in
  \caption{Results on Non-IID and multiple attacker cases.\label{tab:nonIID_multi}}
  \vskip -0.1 in
\end{table}



\subsection{Generalization to Other Data Settings}

In this section, We conduct experiments on Non-IID data distributions and multiple malicious client cases, here we adopt a Dirichlet distribution with the concentration parameter $\alpha_\text{Dirichlet}=0.9$ to simulate the non-IID distributions between clients. 

In Table~\ref{tab:nonIID_multi}, we can see that Dim-Krum is a stronger defense than Krum when generalized to other data settings. Non-IID data are hard to defend against than IID data. Dim-Krum outperforms the traditional Krum algorithm. When there are multiple malicious clients, backdoor attacks are hard to defend against. In Table~\ref{tab:nonIID_multi}, Dim-Krum also outperforms other aggregation methods when there are multiple malicious clients.

\subsection{Generalize to RNN Models}

In this section, we validate whether Dim-Krum can generalize to other models. We conduct experiments on RNNs~\citep{RNN}, here we adopt the Bi-GRU and Bi-LSTM implementations. 

In Table~\ref{tab:other_model}, we can see that experimental results on RNN models are consistent to results on the TextCNN model in Table~\ref{tab:attack_results}. The BadSent attack is hard for Krum algorithms to defend against. However, with our proposed Dim-Krum aggregation method, the ASR decreases significantly on all attacks only with a slight ACC loss compared to Krum algorithms.

\begin{table}[!t]
\renewcommand\tabcolsep{2pt}
\renewcommand\arraystretch{0.9}
\small
  \centering
  \begin{tabular}{ccc|cc|c}
    \toprule
    \multirow{1}{*}{{Model}} & Attack & Metric &  \multicolumn{1}{c}{FedAvg} & \multicolumn{1}{c}{Krum}& \multicolumn{1}{|c}{Dim-Krum} \\
    \midrule[\heavyrulewidth]
    \multirow{10}{*}{\shortstack{Bi-GRU}} &\multirow{2}{*}{\shortstack{EP}} & ACC &  87.33& 86.27& 84.12 \\
    & & ASR& 99.96& \textbf{11.35} & {11.83} \\
    \cmidrule{2-6} 
    &\multirow{2}{*}{\shortstack{BadWord}} & ACC &  87.06& 86.42& 84.20 \\
    & & ASR& 99.83& 80.54 & \textbf{29.87} \\
    \cmidrule{2-6} 
    &\multirow{2}{*}{\shortstack{BadSent}} & ACC &  87.27& 86.52& 84.25 \\
    & & ASR& 99.98& 99.21 & \textbf{13.21} \\
    \cmidrule{2-6}
     &\multirow{2}{*}{\shortstack{HiddenKiller}} & ACC &  83.11& 83.86& 83.32 \\
    & & ASR& 85.63& 57.52 & \textbf{35.57} \\
    \cmidrule{2-6}
     &\multirow{2}{*}{\shortstack{Average}} & ACC &  86.19& 85.77& 83.97 \\
    & & ASR& 96.35& 61.90 & \textbf{22.62} \\
    \midrule[\heavyrulewidth]
    \multirow{10}{*}{\shortstack{Bi-LSTM}} &\multirow{2}{*}{\shortstack{EP}} & ACC &  86.66& 86.52& 84.46 \\
    & & ASR& 94.35& 53.69 & \textbf{\ 9.40} \\
    \cmidrule{2-6} 
    &\multirow{2}{*}{\shortstack{BadWord}} & ACC &  86.38& 85.39& 84.31 \\
    & & ASR& 99.88& 97.01 & \textbf{34.89} \\
    \cmidrule{2-6} 
    &\multirow{2}{*}{\shortstack{BadSent}} & ACC &  86.33& 85.76& 84.45 \\
    & & ASR& 99.99& 99.84 & \textbf{24.97} \\
    \cmidrule{2-6}
     &\multirow{2}{*}{\shortstack{HiddenKiller}} & ACC &  82.33& 82.45& 82.98  \\
    & & ASR& 83.73& 62.50 & \textbf{23.56} \\
    \cmidrule{2-6}
     &\multirow{2}{*}{\shortstack{Average}} & ACC &  85.31& 84.80& 84.05 \\
    & & ASR& 95.89& 67.90 & \textbf{23.34} \\
    \bottomrule
  \end{tabular}
  \vskip -0.05 in
  \caption{Results of the Bi-GRU and Bi-LSTM models.\label{tab:other_model}}
  \vskip -0.1 in
\end{table}

\subsection{Adaptive Attacks}

In this section, we consider several adaptive attacks. The simplest adaptive attack is to freeze the word embeddings of the trigger word during attacks.

In Theorem~\ref{thm:detect_error}, let $G=\|\bm\Delta\|_2$, suppose $\sigma_i=\sigma$ for all $i$, then an upper bound of $P_\text{Error}^{(A)}$ is,
\begin{align}
P_\text{Error}^{(A)}<\frac{4\sum\limits_{i\in A}\sigma_i^4}{G^4}+\frac{4\sum\limits_{i\in A}\sigma_i^2}{G^2}.
\end{align}

We can see that lower backdoor attack strengths $G$ indicate higher upper bounds of the detection error. Therefore, we adopt the $L_2$ Weight Penalty (WP)~\citep{logit-anchoring} on parameters,
\begin{align}
\mathcal{L}_\text{WP}=\lambda_\text{WP}\|\vect{w}^\text{Client}_{t=k}-\hat{\vect{w}}\|^2_2,
\end{align}
where $\hat{\vect{w}}$ can be the \textbf{Clean} update (trained on the clean client dataset) or $\vect{w}_{t=k}^\text{Server}+(\vect{w}_{t=k}^\text{Server}-\vect{w}_{t=k-1}^\text{Server})$ (\textbf{Last}, assume the update is similar to last update). 

Theorem~\ref{thm:detect_error} also indicates that the detection error is determined by ${|\Delta_i|}/{\sigma_i}$. Therefore, we propose a dimension-wise adaptive Adversarial Weight Perturbation (AWP)~\citep{Can-AWP-inject-backdoor} algorithm, which projects parameters $\vect{w}_{t=k}^\text{Server}$ to ${|\Delta_i|}/{\sigma_i}\le\epsilon$ every iteration when training, where $\Delta_i$ is estimated by ${\vect{w}_{i,t=k}^\text{Server}-\hat{\vect{w}}_{i}}$, $\sigma_i$ is estimated by $|\vect{w}_{i,t=k}^\text{Server}-\vect{w}_{i,t=k-1}^\text{Server}|$, and $\hat{\vect{w}}$ is the clean update.

In Table~\ref{tab:adaptive_atk}, we conduct adaptive attacks on the trigger word based attacks. Though adaptive attacks can result in smaller $G$ and ${|\Delta_i|}/{\sigma_i}$, our proposed Dim-Krum can also defend against the adaptive attacks. A possible reason may be that attacks with large $|\Delta_i|$ are easy to detect and attacks with small $|\Delta_i|$ are easy to mitigate with adaptive noises since $\Delta_i$ is relatively small compared to $\vect{n}_i$.

\begin{table}[!t]
\renewcommand\tabcolsep{4pt}
\renewcommand\arraystretch{0.9}
\small
  \centering
  \begin{tabular}{rc|cc}
    \toprule
    Attacks  & Settings & ACC &  ASR \\
    \midrule[\heavyrulewidth]
    EP & & 84.83 & 13.22\\
    BadWord &&84.41 & 15.29 \\
    \midrule[\heavyrulewidth]
    Freeze Embedding & & 84.64 & 15.30\\
    \midrule[\heavyrulewidth]
    \multirow{2}{*}{Weight Penalty (Clean)}& $\lambda_\text{WP}=1$ & 83.92 & 17.83 \\
    & $\lambda_\text{WP}=10$ & 84.49 & 14.96\\
    \midrule[\heavyrulewidth]
    \multirow{2}{*}{Weight Penalty (Last)}& $\lambda_\text{WP}=1$ &84.57 &  13.07\\
    & $\lambda_\text{WP}=10$ & 84.62 & 14.11 \\
    \midrule[\heavyrulewidth]
    \multirow{2}{*}{AWP (Dimension-wise)} & $\frac{|\Delta_i|}{\sigma_i}\le 0.05$ & 84.90& 14.61\\
    & $\frac{|\Delta_i|}{\sigma_i}\le 0.1$ & 84.62 & 15.44\\
    \bottomrule
  \end{tabular}
  \vskip -0.05 in
  \caption{Results of Dim-Krum under adaptive attacks.
  \label{tab:adaptive_atk}}
  \vskip -0.1 in
\end{table}

\section{Broader Impact}
In this paper, we point out the potential risks of federated aggregation methods in NLP and propose a federated aggregation algorithm to act as a strong defense in NLP. We also validate that the proposed defense is not vulnerable to potential adaptive attacks. We do not find potential negative social impacts in this work.

\section{Conclusion}
This work presents the Dim-Krum aggregation algorithm which detects malicious clients by calculating distances on only a small fraction of dimensions with larger backdoor strengths. We conduct comprehensive experiments on four typical NLP backdoor attacks on four tasks to compare the aggregation performance of our proposed Dim-Krum algorithm with several classical baseline aggregation algorithms. Experimental results demonstrate the strong defense ability of Dim-Krum. Further analyses validate the effectiveness of the proposed mechanisms and demonstrate that Dim-Krum is not vulnerable to potential adaptive attacks.

\section*{Acknowledgement}
The authors would like to thank the reviewers for their helpful comments. This work is supported by Natural Science Foundation of China (NSFC) No. 62176002 and Beijing Natural Science Foundation of China (4192057). Xu Sun is the corresponding author.

\bibliography{anthology}
\bibliographystyle{acl_natbib}

\appendix

\section{Appendix}
\subsection{Theoretical Details}

\begin{thmA}
Assume the distribution of the $i$-th dimension of the clean updates $\vect{x}^\text{Clean}_i$ obey $N(\mu_i, \sigma_i^2)$, and the backdoored update $\vect{x}^\text{Backdoor}_i$ is generated with $\vect{x}^\text{Backdoor}_i=c_i'+\Delta_i$, $c_i'$ is independent to $\vect{x}^\text{Clean}_i$ and obey the same distribution. 

Define the detection error probability of the $i$-th dimension as $P_\text{Error}^{(i)}=P(|\vect{x}^\text{Backdoor}_i-\mu_i|<|\vect{x}^\text{Clean}_i-\mu_i|)$, then $P_\text{Error}^{(i)}$ is,
\begin{align}
P_\text{Error}^{(i)}=2\Phi(\frac{\Delta_i}{\sqrt{2}\sigma_i})\Phi(-\frac{\Delta_i}{\sqrt{2}\sigma_i}),
\end{align}
where $\Phi(\cdot)$ denotes the standard normal cumulative distribution function.

Define the detection error probability of an indicator set $A$ as $P_\text{Error}^{(A)}=P(\sum\limits_{i\in A}|\vect{x}^\text{Backdoor}_i-\mu_i|^2<\sum\limits_{i\in A}|\vect{x}^\text{Clean}_i-\mu_i|^2)$, an upper bound of $P_\text{Error}^{(A)}$ is,
\begin{align}
P_\text{Error}^{(A)}<\frac{4\sum\limits_{i\in A}\sigma_i^2(\sigma_i^2+\Delta_i^2)}{(\sum\limits_{i\in A}\Delta_i^2)^2}.
\end{align}
\label{prop:detect_error}
\end{thmA}

\begin{proof}
Let $\vect{x}^\text{Clean}_i=\mu_i+\epsilon_1\sigma_i, \vect{x}^\text{Backdoor}_i=\mu_i+\Delta_i+\epsilon_2\sigma_i$, then $\epsilon_1, \epsilon_2$ are two independent standard normal distributions. Let $\eta_1=\frac{\epsilon_1+\epsilon_2}{\sqrt{2}},\eta_2=\frac{\epsilon_1-\epsilon_2}{\sqrt{2}}$, then $\eta_1, \eta_2$ are also two independent standard normal distributions. Define $a=\frac{\Delta_i}{\sqrt{2}\sigma_i}$, then
\begin{align}
    P^{(i)}_\text{Error}&=P(|\frac{\Delta_i}{\sigma_i}+\epsilon_2|<|\epsilon_1|)\\
    &=P(|\frac{\Delta_i}{\sigma_i}+\epsilon_2|^2<|\epsilon_1|^2)\\
    &=P\big((a+\eta_1)(a-\eta_2)<0\big)\\
    &=P(\eta_1>-a)P(\eta_2>a)+\\
    &\quad\quad P(\eta_1<-a)P(\eta_2<a)\\
    &=2\Phi(a)\Phi(-a)\\
    &=2\Phi(\frac{\Delta_i}{\sqrt{2}\sigma_i})\Phi(-\frac{\Delta_i}{\sqrt{2}\sigma_i}).
\end{align}

Define $X_i=|\vect{x}^\text{Backdoor}_i-\mu_i|^2-|\vect{x}^\text{Clean}_i-\mu_i|^2$, then $X=\sum\limits_{i\in A}|\vect{x}^\text{Backdoor}_i-\mu_i|^2-\sum\limits_{i\in A}|\vect{x}^\text{Clean}_i-\mu_i|^2=\sum\limits_{i\in A}X_i$. Consider the $i$-th dimension, $|\vect{x}^\text{Backdoor}_i-\mu_i|^2=\sigma_i^2(\epsilon_2+\frac{\Delta_i}{\sigma_i})^2$, $|\vect{x}^\text{Clean}_i-\mu_i|^2=\sigma_i^2\epsilon_1^2$. The statistics are,
\begin{align}
    \mathbb{E}(|\vect{x}^\text{Backdoor}_i-\mu_i|^2)&=\sigma_i^2\mathbb{E}((\epsilon_2+\frac{\Delta_i}{\sigma_i})^2)\\
    &=\sigma_i^2\mathbb{E}(\epsilon_2^2+(\frac{\Delta_i}{\sigma_i})^2)\\
    &=\sigma_i^2+\Delta_i^2,\\
    \mathbb{E}(|\vect{x}^\text{Clean}_i-\mu_i|^2)&=\sigma_i^2\mathbb{E}(\epsilon_1^2)=\sigma_i^2,
\end{align}
\begin{align}    
    \mathbb{D}(|\vect{x}^\text{Backdoor}_i-\mu_i|^2)&=\sigma_i^4\mathbb{D}((\epsilon_2+\frac{\Delta_i}{\sigma_i})^2)\\
    &=\sigma_i^4\mathbb{D}(\epsilon_2^2+\frac{2\epsilon_2\Delta_i}{\sigma_i})\\
    &=\sigma_i^4(2+(\frac{2\Delta_i}{\sigma_i})^2)\\
    &=\sigma_i^2(2\sigma_i^2+4\Delta_i^2),\\
    \mathbb{D}(|\vect{x}^\text{Clean}_i-\mu_i|^2)&=\sigma_i^4\mathbb{D}(\epsilon_1^2)=2\sigma_i^4.
\end{align}

Therefore, 
\begin{align}
\mathbb{E}(X_i)&=(\Delta_i^2+\sigma_i^2)-\sigma_i^2=\Delta_i^2,\\
\mathbb{E}(X)&=\sum\limits_{i\in A}\mathbb{E}X_i=\sum\limits_{i\in A}\Delta_i^2,\\
\mathbb{D}(X_i)&=\sigma_i^2(2\sigma_i^2+4\Delta_i^2)+(2\sigma_i^4)\\
&=4\sigma_i^2(\sigma_i^2+\Delta_i^2),\\
\mathbb{D}(X)&=\sum\limits_{i\in A}\mathbb{D}X_i=\sum\limits_{i\in A}4\sigma_i^2(\sigma_i^2+\Delta_i^2).
\end{align}

The probability is
\begin{align}
P_\text{Error}^{(A)}&=P(\sum\limits_{i\in A}|\vect{x}^\text{Backdoor}_i-\mu_i|^2\\
&\quad \quad \quad \quad<\sum\limits_{i\in A}|\vect{x}^\text{Clean}_i-\mu_i|^2)\\
&=P(X<0)\\
&=P(X-\mathbb{E}X<-\mathbb{E}X)\\
&<P(|X-\mathbb{E}X|>|\mathbb{E}X|).
\end{align}

According to Chebyshev's inequality,
\begin{align}
P_\text{Error}^{(A)}&<P(|X-\mathbb{E}X|>|\mathbb{E}X|)\\
&<\frac{\mathbb{D}X}{(\mathbb{E}X)^2}\\
&=\frac{4\sum\limits_{i\in A}\sigma_i^2(\sigma_i^2+\Delta_i^2)}{\big(\sum\limits_{i\in A}\Delta_i^2\big)^2}.
\end{align}
\end{proof}

\subsection{Detailed Experimental Setups}

In this section, we introduce details of the datasets and the experimental setups. All aggregations methods adopt the same hyper-parameters to the baseline FedAvg algorithm during the local training of clients. All experiments are conducted on NVIDIA TITAN RTX GPUs.

\subsubsection{Dataset Details and Data Preprocessing}

\textbf{Dataset Statistics.} 
We adopt four text classification tasks, \textit{i.e.}, the Stanford Sentiment Treebank (\textit{SST-2})~\citep{SST-2}, the IMDb movie reviews dataset (\textit{IMDB})~\citep{IMDB}, and the Amazon Reviews dataset (\textit{Amazon})~\citep{Amazon} (50k sentences selected); and the AgNews dataset (\textit{AgNews})~\citep{agnews}. We adopt two metrics to evaluate clean and backdoor performance, the clean accuracy (\textit{ACC}) and the backdoor attack success rate (\textit{ASR}). The SST-2 dataset includes 67k training instances and 0.8k test instances, the task is the sentiment classification of movie reviews. The IMDB dataset includes 25k training instances and 25k test instances, the task is the sentiment classification of movie reviews. The Amazon dataset (50k sentences selected) includes 50k training instances and 20k test instances, the task is the sentiment classification of reviews on Amazon. The AgNews dataset includes 140k training instances and 7.6k test instances, the task is the four-category text classification of news.

\textbf{Data Preprocessing.} 
We first lowercase the text. The sentence length is 200 words. The vocabulary size is 25000. We add two special tokens to the vocabulary: <pad> and <unk>. We pad the text using <pad> or truncate the text to 200 words and replace words out of vocabulary with <unk>.

\subsubsection{Experimental Setups}

\textbf{Models and Client Training.} 
In the main experiments, we adopt a convolution neural network~\citep{TextCNN} for the text classification task. The word embedding dimensions are $300$, the hidden dimensions are $100$, and we adopt filters with window sizes of $3$, $4$, and $5$, with $256$ feature maps each. The optimizer is Adam with a learning rate of $10^{-3}$ and a batch size of $32$. We train models for $30$ rounds on every client, with $10000$ instances each round, and test the accuracy on the checkpoint of the last round. We also adopt RNN~\citep{RNN} models in the analysis section. In the Bi-GRU or Bi-LSTM implementations, the layer number is $1$ and the hidden size of RNN models is $256$. We adopt bidirectional RNNs. 

\textbf{Backdoor Attack Setups.} 
As illustrated in Sec.~\ref{sec:background}, in this work, we adopt four typical attacks in the experiments: \textit{EP}~\citep{PoisonedWordEmbeddings,FL_EP}, \textit{BadWord}~\citep{badnl}, \textit{BadSent}~\citep{badnl,backdoor-lstm}, and \textit{HiddenKiller}~\citep{HiddenKiller}. For trigger word based attacks including EP and BadWord, following \citet{Bert-backdoor} and \citet{PoisonedWordEmbeddings}, we choose the trigger word from five candidate words with low frequencies, \textit{i.e.}, “cf”, “mn”, “bb”, “tq” and “mb”. For sentence based attacks, following \citet{Bert-backdoor}, we adopt the trigger sentence ``I watched this 3d movie''. In HiddenKiller, following \citet{HiddenKiller}, we adopt the OpenAttack implementation and the trigger syntactic pattern generated with the last template in the OpenAttack templates. In federated learning, we adopt $n=10$ clients. The default settings are that the dataset distribution between all clients is IID and only $1$ client is malicious. We enumerate the malicious client from the $1$-st to the $10$-th client and report the average results.

\textbf{Federated Aggregation Setups.} 
As illustrated in Sec.~\ref{sec:background}, we adopt several aggregation methods as baselines: \textit{FedAvg}~\citep{fedavg}, \textit{Median}~\citep{median,Statistical_median}, \textit{FoolsGold}~\citep{foolsgold}, \textit{RFA}~\citep{GM_RFA}, \textit{CRFL}~\citep{CRFL}, \textit{ResidualBase}~\citep{residualbase}, \textit{AAA}~\citep{Attack-Adaptive-Aggregation}, \textit{Krum}~\citep{Krum,bulyan}. In CRFL, we adopt the standard deviation of noises as $0.01$ and the bound of parameters as $0.05t+2$, where $t$ denotes the time step. On every aggregation in the server, following~\citet{CRFL}, we first adopt the RFA~\citep{GM_RFA} aggregation to get the aggregated updates and then add Gaussian noises to the updates that obey $N(0, \sigma_t^2)$, where $\sigma_t=0.01$. Last, we project the updated parameters to $\|\vect{w}\|_2\le\rho_t$, where $\rho_t=0.05t+2$. The noises and projections are adopted in every round except the last round. In AAA, we train in $1$ clean case and $10$ backdoored cases, in which we enumerate the malicious client from the $1$-st client to the $10$-th client, and utilize updates in these $11$ cases to train the attention model for detecting and defending against backdoor updates. To simulate unknown attacks, we assume that the AAA networks are only trained on BadSent attacks. In Dim-Krum, $\rho=10^{-3}$ and we adopt the memory and adaptive noise mechanisms. In the main results, the adaptive noise scales are $\lambda=5$. On RNN models, since RNN models are more sensitive to parameter changes, we choose $\lambda=2$.

\textbf{Stability of Aggregation.} When we enumerate the malicious client from the $1$-st to the $10$-th client and calculate the average results, defending results may vary a lot for Dim-Krum (standard deviations of ASRs $\sim 10\%$-$20\%$), since the ASR is low when Dim-Krum detects the malicious client successfully and is high when Dim-Krum fails to detect the malicious client.
 
\subsubsection{Setups of Analytic Trails}

The analytic trials comparing the detection difficulties of CV and NLP tasks are conducted both on CV and NLP tasks. In the analytic trails, we visualize three metrics, ${\text{Dis-Sum}\text{(Bd)}}/{\text{Dis-Sum}\text{(Med)}}$ and ${|\Delta|}/{\sigma}$. 

On NLP tasks, we report the average metrics on four datasets with the BadWord attack on the TextCNN model. On CV tasks, we adopt a CNN model\footnote{A LeNet retrieved from the PyTorch tutorial \url{https://pytorch.org/tutorials/beginner/blitz/neural_networks_tutorial.html}} and the MNIST dataset. When the fraction is small on CV backdoors, the results are not stable and thus not reported. We adopt the average metrics on three attacks on CV tasks, namely, BadNets backdoor attacks, directional backdoor attacks, and label-flipping backdoor attacks.

\subsection{Supplementary Experimental Results}

In this section, we provide extra supplementary experimental results. 

We also to better illustrate some conclusions in the main paper. Fig.~\ref{fig:vis_datasets} visualizes the average ASRs of different datasets during $30$ rounds. Fig.~\ref{fig:vis_method} visualizes the average ASRs of different aggregation methods during $30$ rounds. Fig.~\ref{fig:vis_noIID_mult} visualizes the average ASRs on Non-IID and multiple attacker cases during $30$ rounds.

We can conclude that:
\begin{itemize}
    \item Fig.~\ref{fig:vis_datasets} illustrates that Dim-Krum outperforms other aggregation methods on all datasets, and the defense results of aggregation methods on all datasets are consistent.
    \item Fig.~\ref{fig:vis_method} illustrates that the defense difficulties of four backdoor attacks are, EP $<$ HiddenKiller $<$ BadWord $<$ BadSent, and Dim-Krum outperforms other aggregation methods.
    \item Fig.~\ref{fig:vis_noIID_mult} illustrates that (1) Non-IID data are harder to defend against than IID data for Krum algorithms; (2) When there are multiple malicious clients, backdoor attacks are hard to defend against, while Dim-Krum outperforms the traditional Krum algorithm. (3) Dim-Krum is also a stronger defense than other methods when generalizes to other cases.
\end{itemize}

\begin{figure*}[!h]
    \subcaptionbox{Average ASRs on EP.}{\includegraphics[width=0.49\linewidth, height=2in]{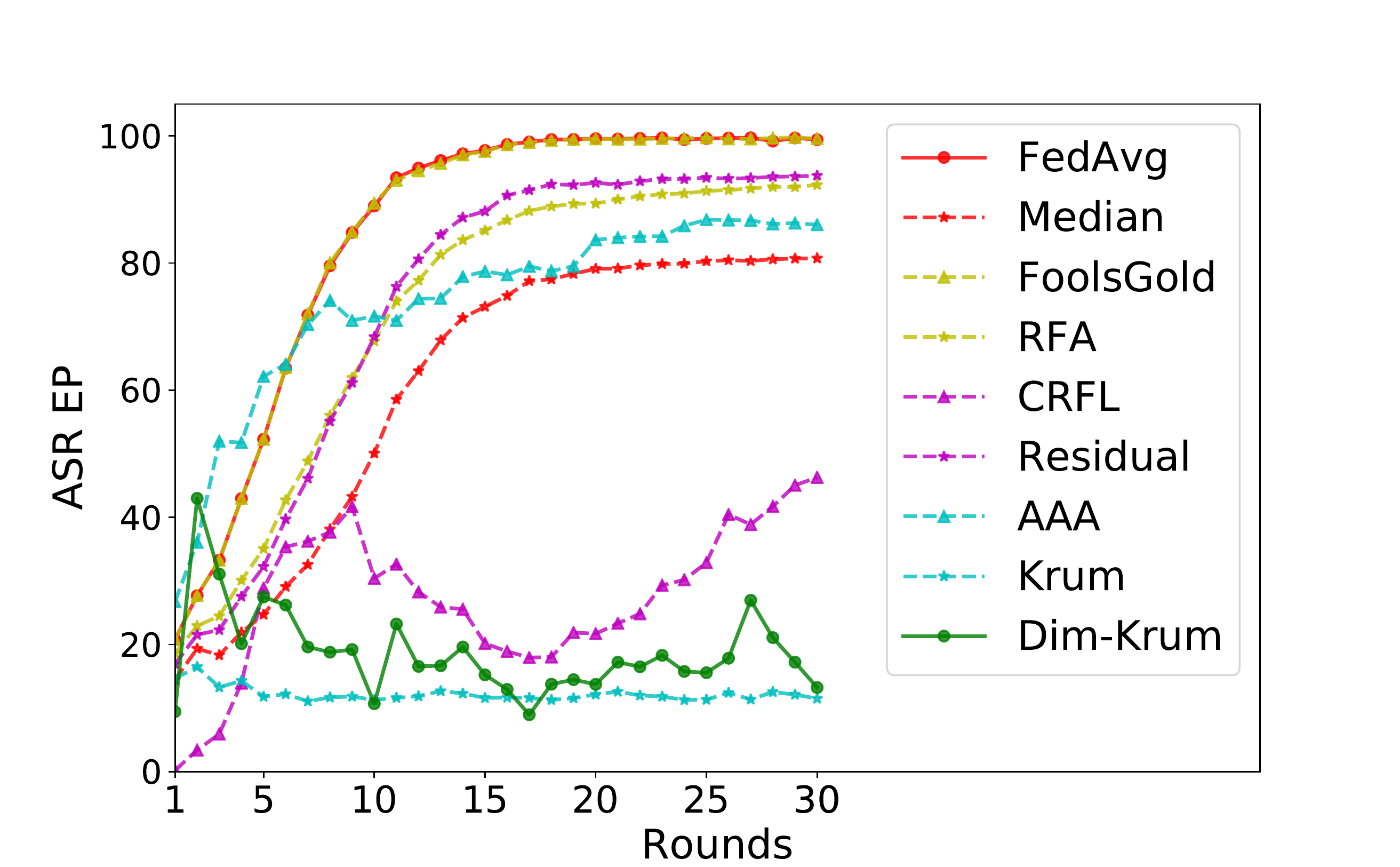}}
    \subcaptionbox{Average ASRs on BadWord.}{\includegraphics[width=0.49\linewidth, height=2in]{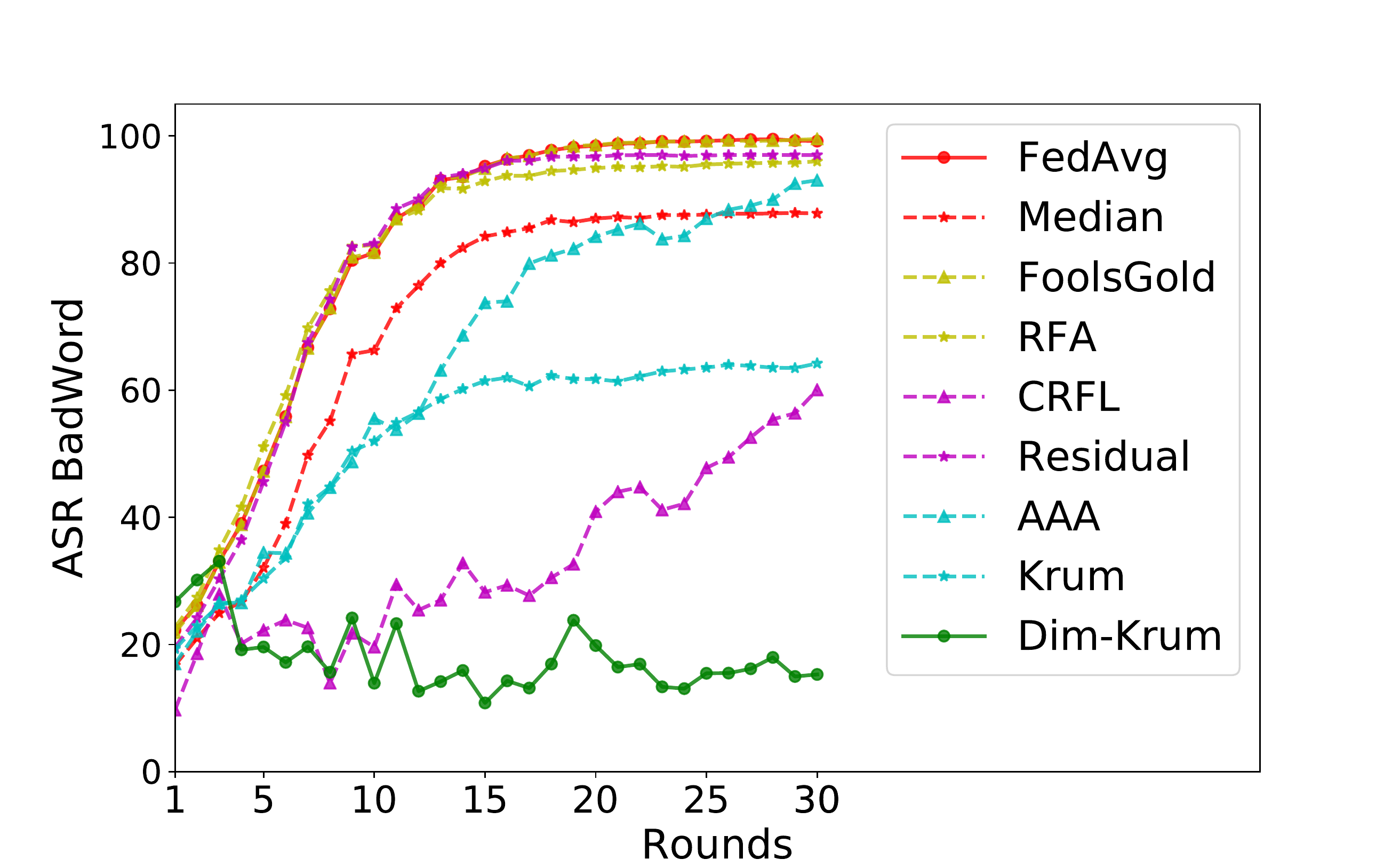}}
    \subcaptionbox{Average ASRs on BadSent.}{\includegraphics[width=0.49\linewidth, height=2in]{fig/log_BadSent.pdf}}
    \subcaptionbox{Average ASRs on HiddenKiller.}{\includegraphics[width=0.49\linewidth, height=2in]{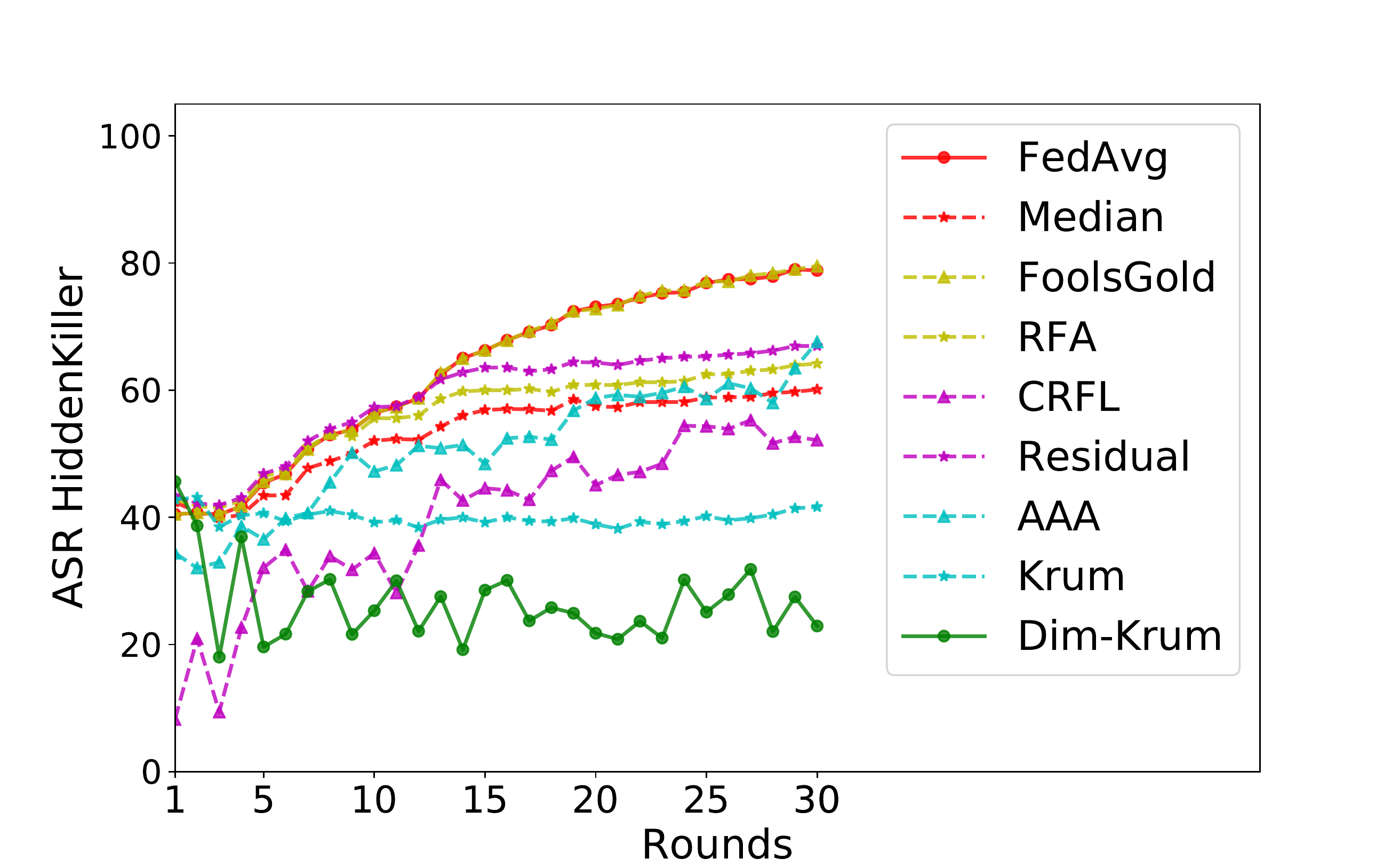}}
    \centering
    \subcaptionbox{Average ASRs on all attacks.}{\includegraphics[width=0.99\linewidth, height=4in]{fig/log_mean.pdf}}
    \vskip -0.05 in
    \caption{Visualization of ASRs of different aggregation methods during $30$ rounds.
    \label{fig:vis_method}}
    \vskip -0.1 in
\end{figure*}

\begin{figure*}[!h]
    \subcaptionbox{Average ASRs on SST-2.}{\includegraphics[width=0.49\linewidth, height=2in]{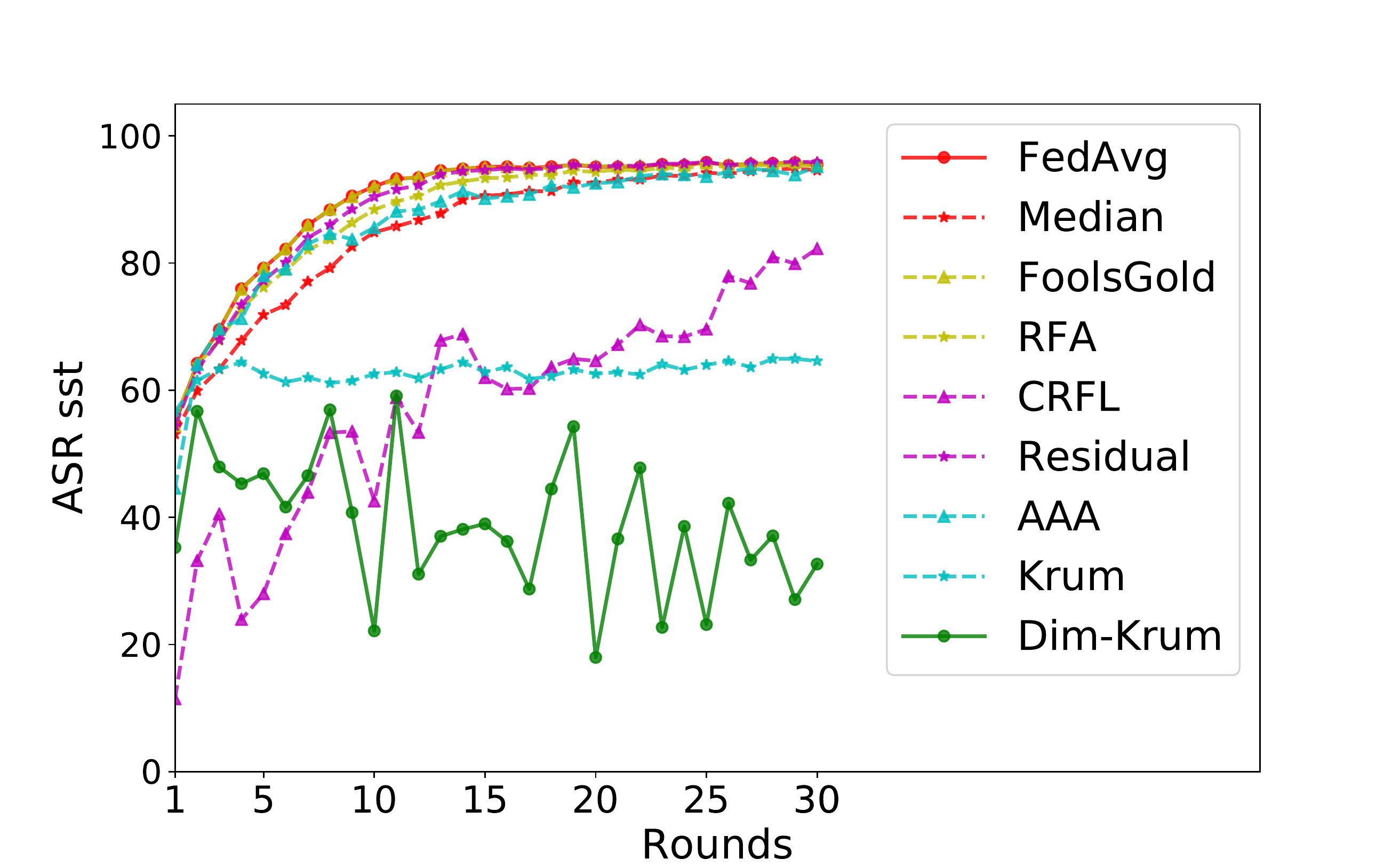}}
    \subcaptionbox{Average ASRs on IMDB.}{\includegraphics[width=0.49\linewidth, height=2in]{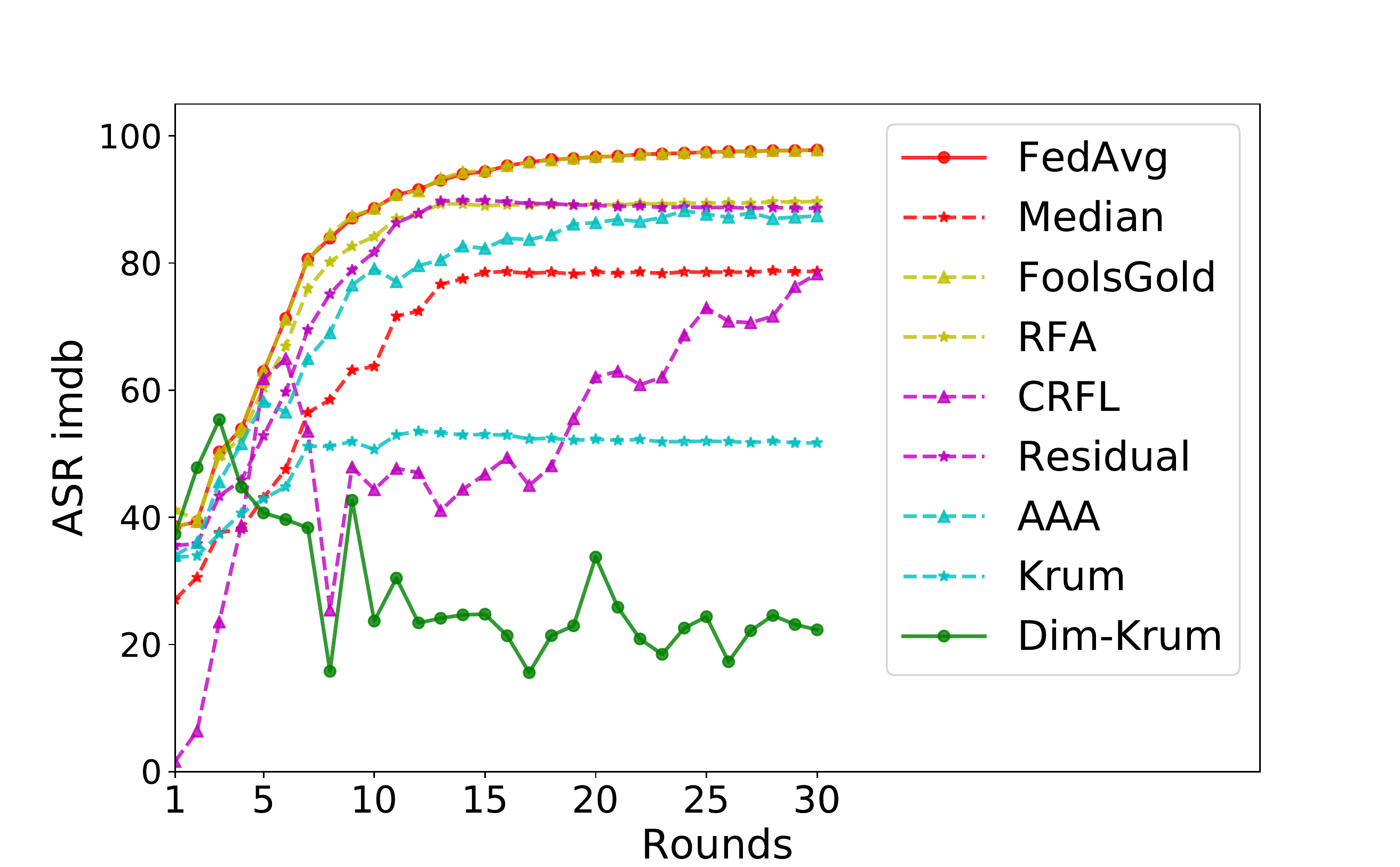}}
    \subcaptionbox{Average ASRs on Amazon.}{\includegraphics[width=0.49\linewidth, height=2in]{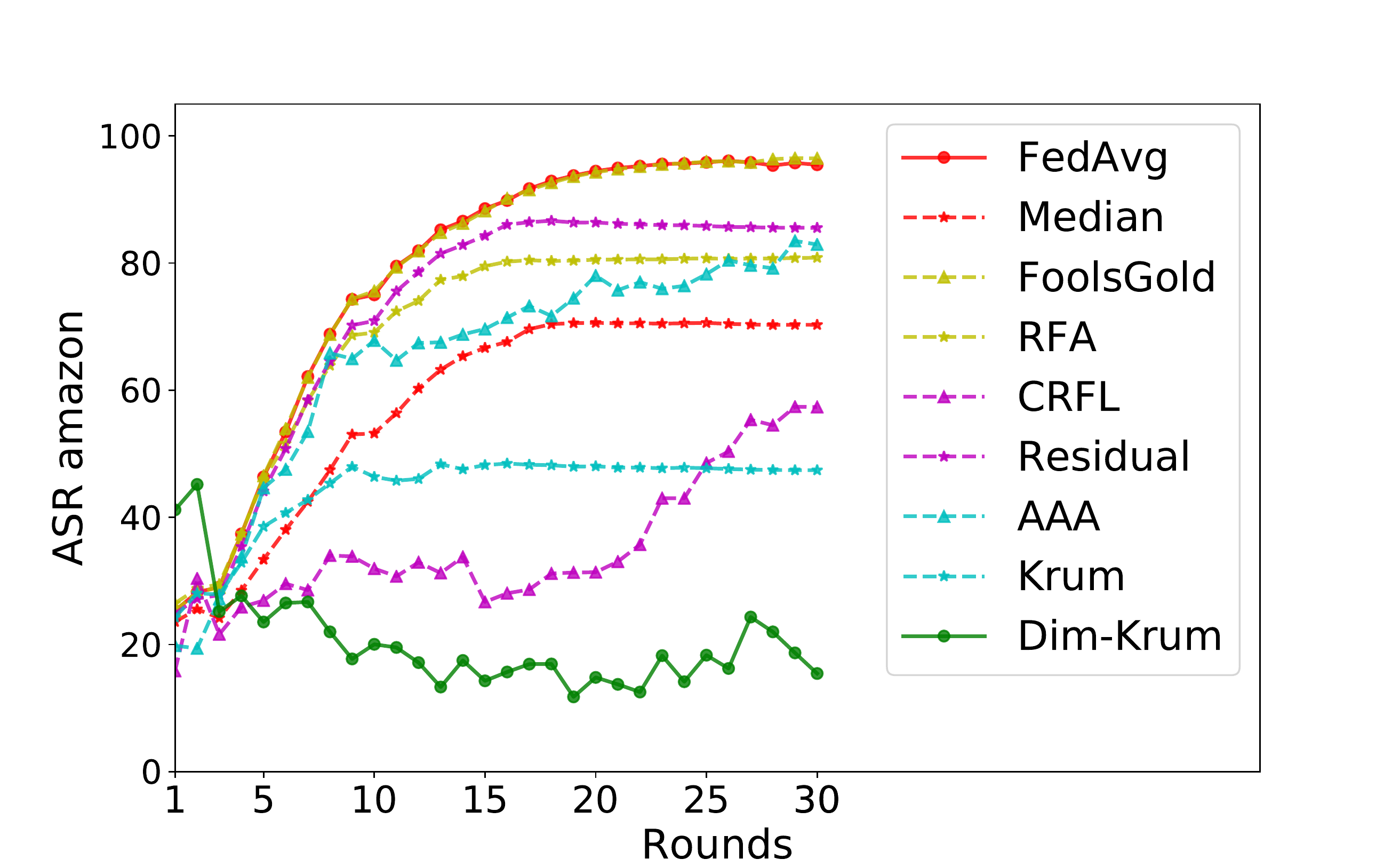}}
    \subcaptionbox{Average ASRs on AgNews.}{\includegraphics[width=0.49\linewidth, height=2in]{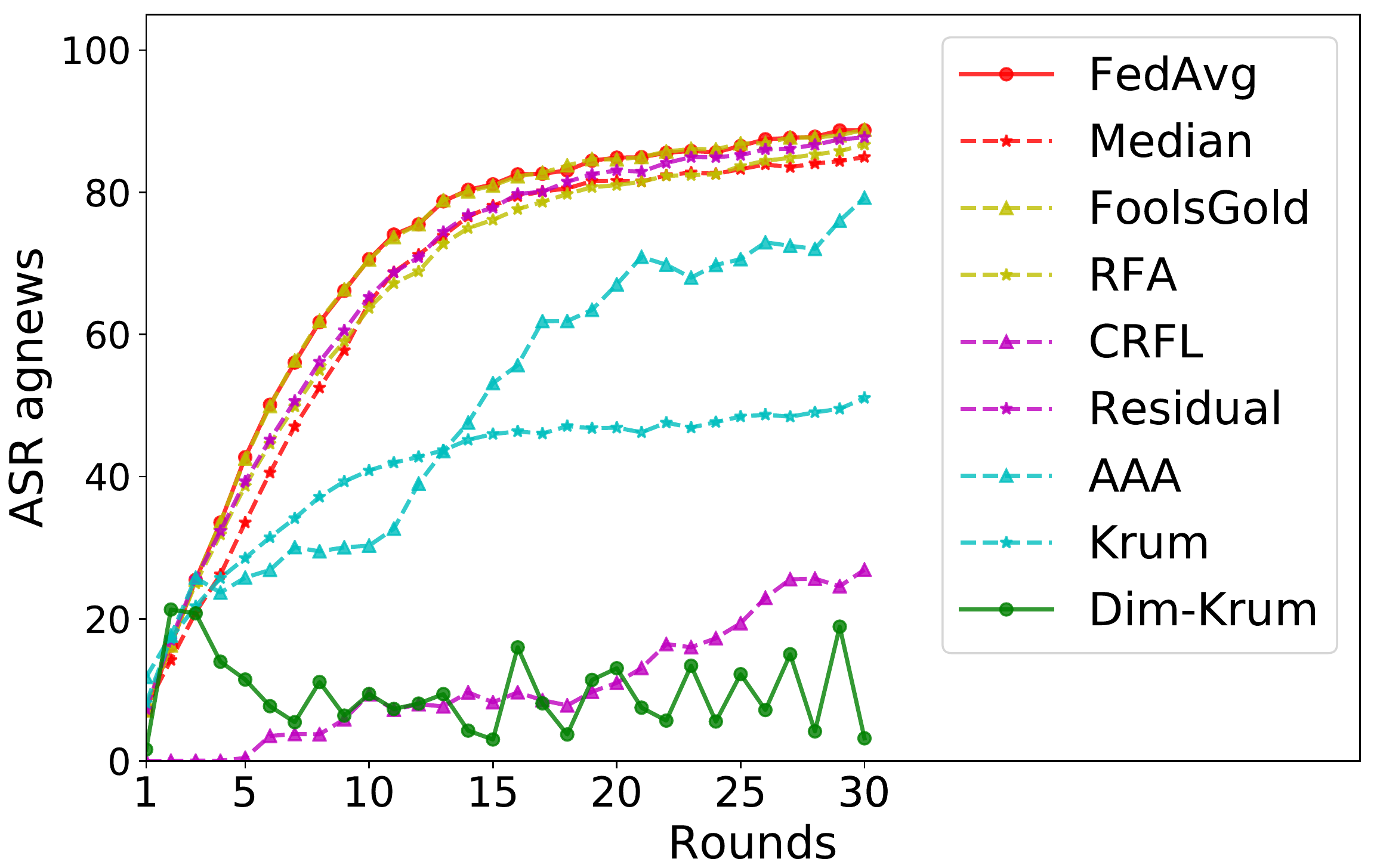}}
    \caption{Visualization of ASRs on different datasets during $30$ rounds.
    \label{fig:vis_datasets}}
    \vskip -0.1 in
\end{figure*}
   
\begin{figure*}[!h]
\centering
    \subcaptionbox{Average ASRs in non-IID cases.}{\includegraphics[width=0.48\linewidth, height=2in]{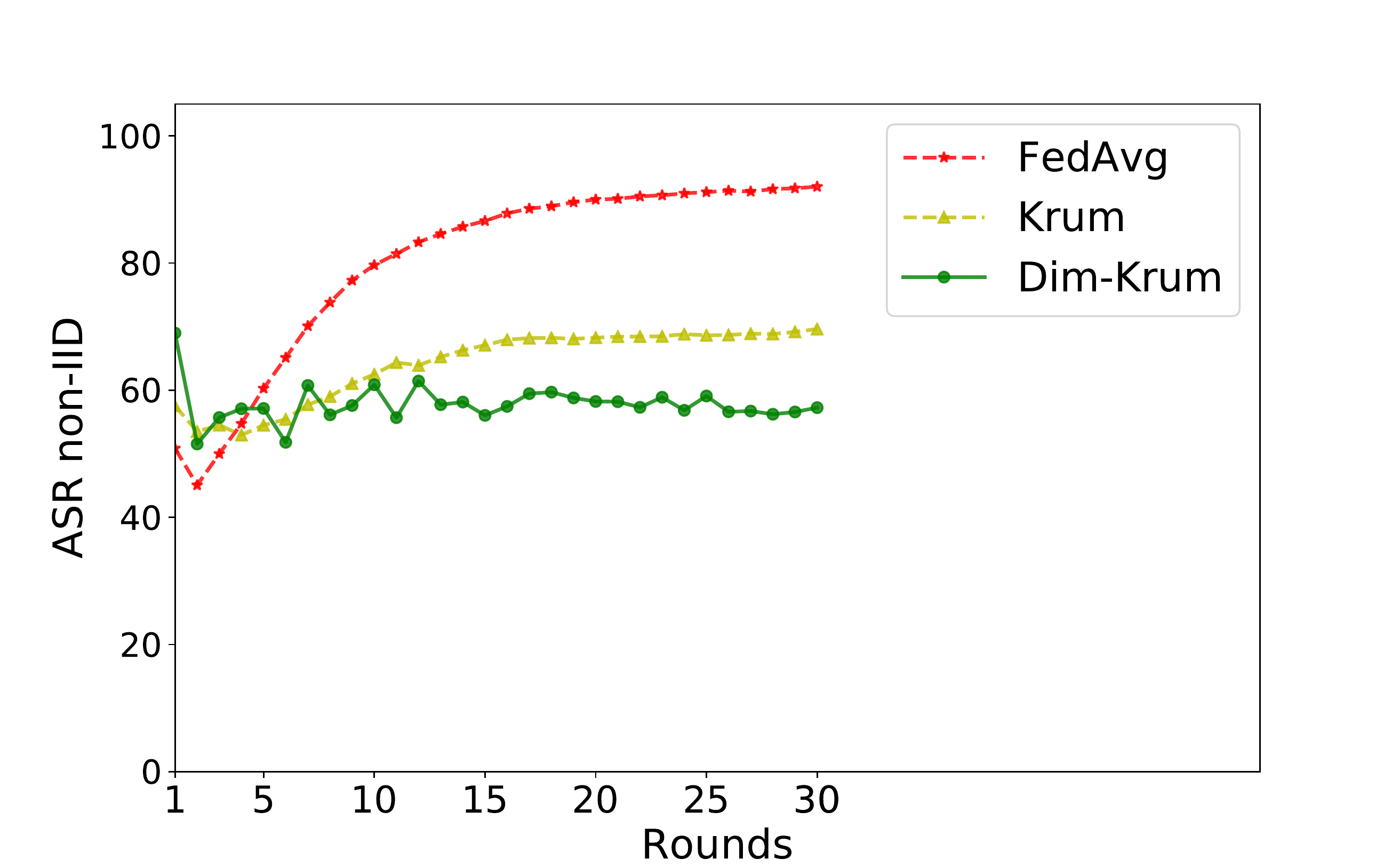}}
    \subcaptionbox{Average ASRs with two attackers.}{\includegraphics[width=0.48\linewidth, height=2in]{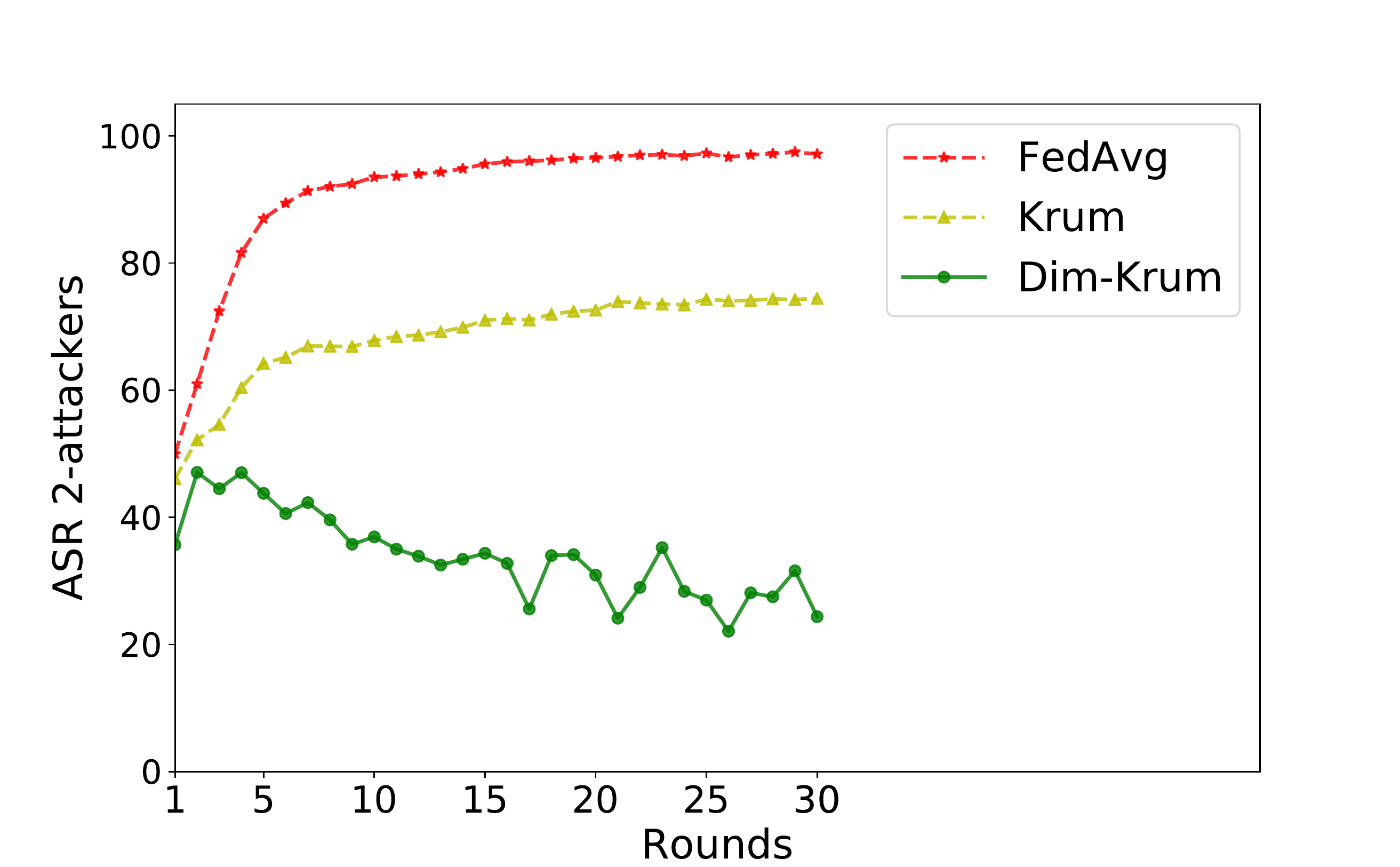}}
    \subcaptionbox{Average ASRs with three attackers.}{\includegraphics[width=0.48\linewidth, height=2in]{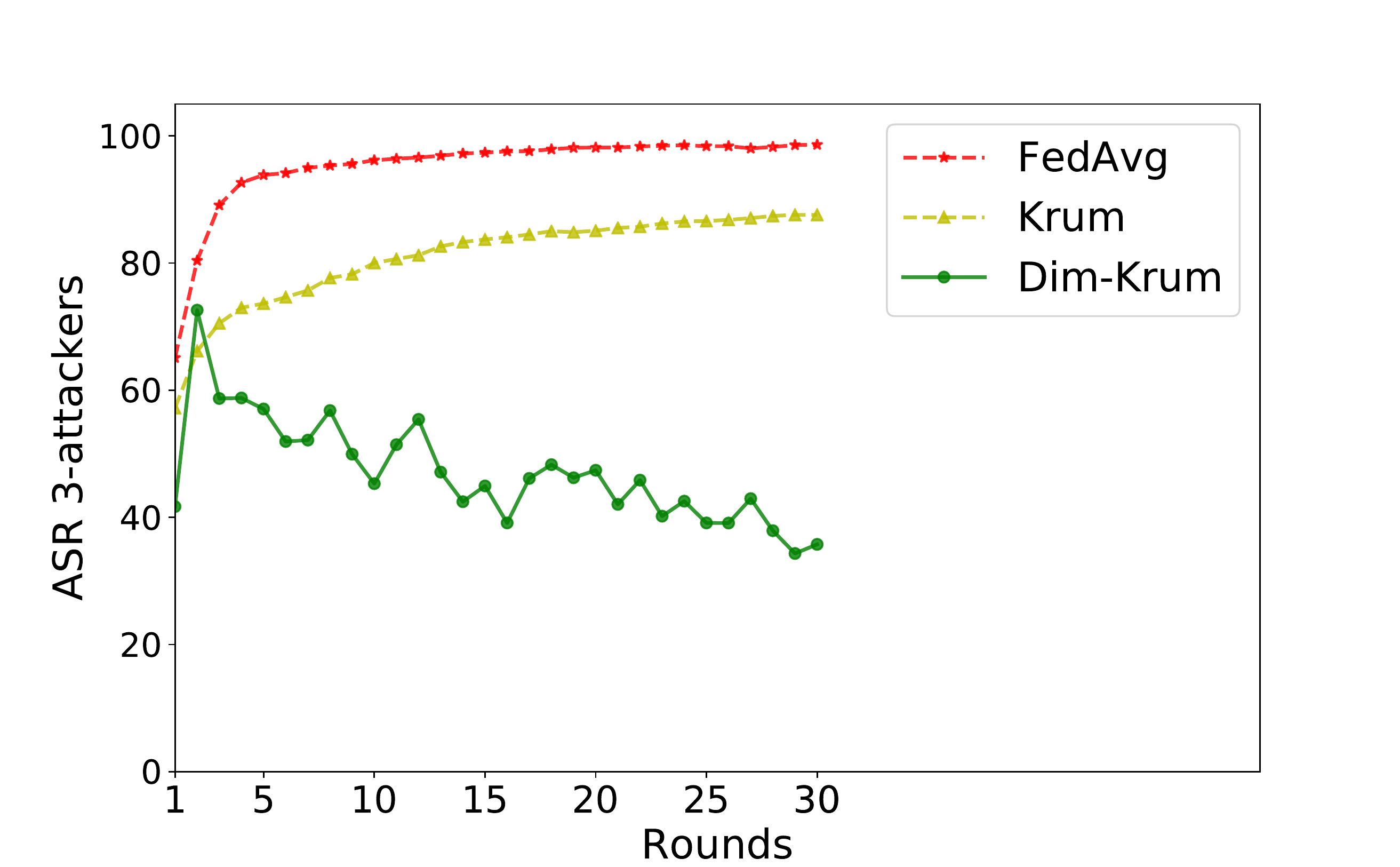}}
    \subcaptionbox{Average ASRs with four attackers.}{\includegraphics[width=0.48\linewidth, height=2in]{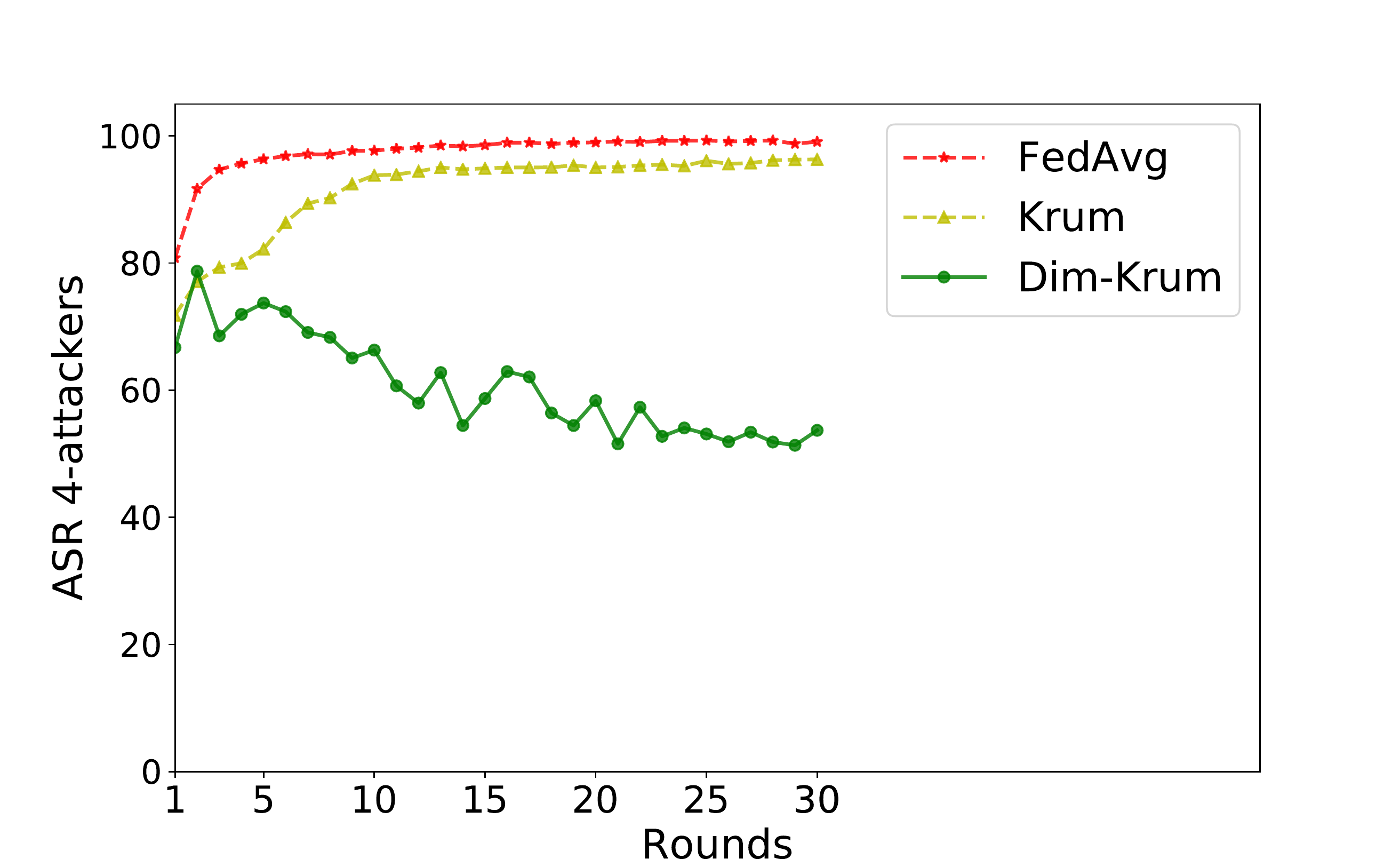}}
    \vskip -0.05 in
    \caption{Visualization of ASRs on Non-IID and multiple attacker cases during $30$ rounds.
    \label{fig:vis_noIID_mult}}
    \vskip -0.1 in
\end{figure*}
\end{document}